\documentclass[journal]{IEEEtran}

\usepackage{times}
\usepackage{epsfig}
\usepackage{graphicx}
\usepackage{amsmath}
\usepackage{amssymb}
\usepackage{amsthm}
\usepackage{setspace}
\usepackage{algorithm}
\usepackage{algorithmic}
\usepackage{color}
\usepackage{bm}
\usepackage{algorithm}
\usepackage{algorithmic}
\usepackage{booktabs}
\usepackage{multirow}

\usepackage[pagebackref=false,breaklinks=true,letterpaper=true,colorlinks,bookmarks=false]{hyperref}

\theoremstyle{plain}
\newtheorem{theorem}{Theorem}
\newtheorem{lemma}[theorem]{Lemma}

\begin{document}

\title{Ensemble Teaching For Hybrid Label Propagation}
%
%
%

\author{Chen~Gong,~\IEEEmembership{Member,~IEEE,}
        Dacheng~Tao,~\IEEEmembership{Fellow,~IEEE,}
        Xiaojun~Chang,~\IEEEmembership{}
        and~Jian~Yang,~\IEEEmembership{Member,~IEEE}

\thanks{This research is supported by NSF of China (No: 61602246, 61572315, 91420201, 61472187, 61502235, 61233011 and 61373063), 973 Plan of China (No: 2014CB349303 and 2015CB856004), Program for Changjiang Scholars, NSF of Jiangsu Province (No: BK20171430), the Open Project Program (No. MJUKF201723) of Fujian Provincial Key Laboratory of Information Processing and Intelligent Control
(Minjiang University), and Australian Research Council Discovery Project (No: FL-170100117, DP-140102164, LP-150100671). \emph{(Corresponding author: Chen Gong (chen.gong@njust.edu.cn))}.}
\thanks{C. Gong is with the School of Computer Science and Engineering, Nanjing University of Science and Technology, Nanjing 210094, China, and also with Fujian Provincial Key laboratory of Information Processing and Intelligent Control (Minjiang University), Fuzhou 350108, China.}
\thanks{J. Yang is with the School of Computer Science and Engineering, Nanjing University of Science and Technology, China, 210094.}
\thanks{D. Tao is with the UBTECH Sydney Artificial Intelligence Centre and the School of Information Technologies in the Faculty of Engineering and Information Technologies at University of Sydney, 6 Cleveland St, Darlington, NSW 2008, Australia.}
\thanks{X. Chang is with the Language Technologies Institute, Carnegie Mellon University.}
\thanks{\textcopyright20XX IEEE. Personal use of this material is permitted. Permission from IEEE must be obtained for all other uses, in any current or future media, including reprinting/republishing this material for advertising or promotional purposes, creating new collective works, for resale or redistribution to servers or lists, or reuse of any copyrighted component of this work in other works.}

}


\maketitle

\begin{abstract}
  Label propagation aims to iteratively diffuse the label information from labeled examples to unlabeled examples over a similarity graph. Current label propagation algorithms cannot consistently yield satisfactory performance due to two reasons: one is the instability of single propagation method in dealing with various practical data, and the other one is the improper propagation sequence ignoring the labeling difficulties of different examples. To remedy above defects, this paper proposes a novel propagation algorithm called \underline{Hy}brid \underline{D}iffusion under \underline{En}semble \underline{T}eaching (HyDEnT). Specifically, HyDEnT integrates multiple propagation methods as base ``learners'' to fully exploit their individual wisdom, which helps HyDEnT to be stable and obtain consistent encouraging results. More importantly, HyDEnT conducts propagation under the guidance of an ensemble of ``teachers''. That is to say, in every propagation round the simplest curriculum examples are wisely designated by a teaching algorithm, so that their labels can be reliably and accurately decided by the learners. To optimally choose these simplest examples, every teacher in the ensemble should comprehensively consider the examples' difficulties from its own viewpoint, as well as the common knowledge shared by all the teachers. This is accomplished by a designed optimization problem, which can be efficiently solved via the Block Coordinate Descent (BCD) method. Thanks to the efforts of the teachers, all the unlabeled examples are logically propagated from simple to difficult, leading to better propagation quality of HyDEnT than the existing methods. Experiments on six popular datasets reveal that HyDEnT achieves the highest classification accuracy when compared with six state-of-the-art propagation methodologies such as Harmonic Functions, Fick's Law Assisted Propagation, Linear Neighborhood Propagation, Semi-supervised Ensemble Learning, Bipartite Graph-based Consensus Maximization, and Teaching-to-Learn and Learning-to-Teach.
\end{abstract}


\IEEEpeerreviewmaketitle

\section{Introduction}
\label{sec:intro}
\IEEEPARstart{L}{abel} propagation is an important technique in semi-supervised learning \cite{Goldberg09introductionto,Dornaika2016Learning}. Given an undirected weighted graph, the target of label propagation is to iteratively transfer class labels from labeled examples to unlabeled examples so that the unlabeled examples can be accurately classified. Label propagation is a transduction problem \cite{gammerman1998learning}, which means that we are interested in the classification of a particular set of examples rather than a general decision function for classifying the future unseen examples. In other words, the unlabeled examples to be classified by a label propagation algorithm are available in advance, and they are directly assigned labels without the aid of explicit decision function. Due to its encouraging performance and solid theoretical foundation, label propagation has been applied to various computer vision tasks such as saliency detection \cite{gong2015saliency}, image annotation \cite{cao2008annotating}, image segmentation \cite{wang2009linear}, \emph{etc}.\par

Mathematically, label propagation can be described as follows. Given a labeled set $\mathcal{L}=\{\mathbf{x}_1, \mathbf{x}_2,\cdots, \mathbf{x}_l\}$ containing $l$ labeled examples $\{\mathbf{x}_i\}_{i=1}^{l}\in\mathbb{R}^d$ ($d$ is the dimensionality of every example) and unlabeled set $\mathcal{U}=\{\mathbf{x}_{l+1}, \mathbf{x}_{l+2},\cdots, \mathbf{x}_{l+u}\}$ of size $u$, we may build a weighted similarity graph $\mathcal{G}=\left\langle \mathcal{V},\mathcal{E} \right\rangle$ where $\mathcal{V}$ is the node set representing all $n=l+u$ examples, and $\mathcal{E}$ is the edge set encoding the pairwise relationship between these examples. Then the target of label propagation is to iteratively propagate the labels from $\mathcal{L}$ to $\mathcal{U}$ so that all the examples in $\mathcal{U}$ can be assigned correct labels. However, the results generated by existing methods such as \cite{wang2013dynamic,wang2009linear,Zhu03semisupervisedlearning, Zhou03learningwith} are often far from perfect. We consider that there are two main reasons contributing to this phenomenon. Firstly, the strength of one propagation algorithm is very limited and there does not exist a propagation method that can perfectly handle all the practical situations. For example, the propagation process can be misled by the outliers or ``bridge points'' \cite{gong2015deformed}, therefore utilizing only one method is not reliable for achieving accurate propagation. Secondly, the propagation sequence adopted by existing methods is completely governed by the connectivity among examples in the graph, namely the label information will be transferred from one example to another as long as there is an edge between them. This propagation sequence is sometimes problematic because it does not explicitly consider the propagation difficulty or reliability of different unlabeled examples. Above two shortcomings are very likely to incur error-prone propagations and impair the final performance.\par

In order to address the above two defects that are ubiquitous in the current label diffusion methodologies, this paper proposes a new propagation algorithm called \textbf{Hy}brid \textbf{D}iffusion under \textbf{En}semble \textbf{T}eaching (dubbed ``HyDEnT''). Specifically, to remedy the first shortcoming, we combine multiple existing propagation algorithms in a hybrid way to accomplish reliable propagation. As a result, each of the incorporated base classifiers will give full play to its ability and meanwhile complement to others for obtaining satisfactory performance. In order to deal with the second shortcoming, we regard all the involved propagation algorithms as ``learners'' and associate each of them with a ``teacher'', so that the entire diffusion process is guided by the ensemble of teachers and thus an optimized learning sequence can be generated. Particularly, we assume that different examples have different levels of difficulty, and the teachers should consider the learners' dynamic learning performance to design a suitable propagation sequence, so that all the unlabeled examples in $\mathcal{U}$ are logically classified from simple to difficult. In each propagation round, the simplest examples (\emph{i.e.} a curriculum) agreed by the ensemble of teachers are designated to the learners, and these simplest examples are decided by comprehensively considering both individuality and consistency of all the teachers. This modification to the widespread propagation strategy facilitates the classification of subsequent complex examples by using the accumulated knowledge from the previously propagated simple examples.\par

\begin{figure*}
  \centering
  \includegraphics[width=0.95 \linewidth]{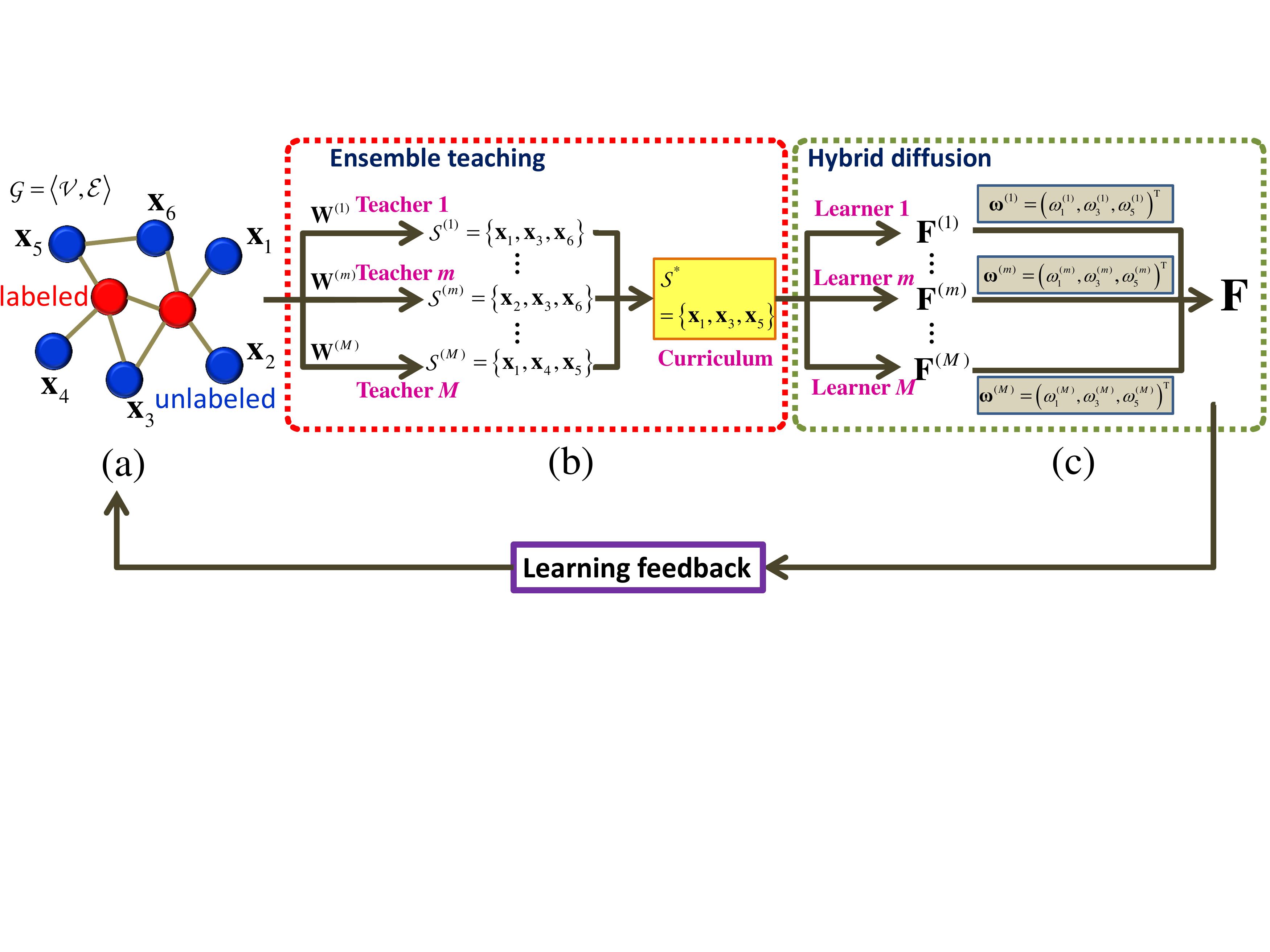}\\
  \caption{The framework of our HyDEnT algorithm. There are totally $M$ learners and $M$ corresponding teachers. In (a), a graph $\mathcal{G}$ is built in which the labeled examples and unlabeled examples $\mathbf{x}_1\sim\mathbf{x}_6$ are represented by red and blue nodes, respectively. In (b), the $M$ teachers first pick up the simplest curriculum examples from $\mathbf{x}_1\sim\mathbf{x}_6$ from their own viewpoint (\emph{i.e.} $\mathcal{S}^{(1)}\sim\mathcal{S}^{(M)}$), and then compromise to a consistent result $\mathcal{S}^{*}$ via the designed ensemble teaching algorithm. In (c), all the selected simplest examples are classified by $M$ different learners, and the obtained label matrices are $\mathbf{F}^{(1)}, \mathbf{F}^{(2)},\cdots, \mathbf{F}^{(M)}$. They are further integrated to the final result $\mathbf{F}$ based on the weighting vectors $\bm{\omega}^{(m)}$ ($m=1,2,\cdots,M$), in which the element ${\omega}_j^{(m)}$ ($j=1,3,5$ in this example) denotes the weight of the $m$-th classifier's decision on the curriculum example $\mathbf{x}_j$. After ``learning'' the curriculum examples in this round, the learners deliver a learning feedback to the ensemble of teachers to help them adaptively establish the suitable curriculum in the next round.}\label{fig:framework}
\vspace{-5pt}
\end{figure*}

The framework of the proposed algorithm is presented in Fig.~\ref{fig:framework}. The labeled examples (two red balls in Fig.~\ref{fig:framework}(a)) and unlabeled examples $\mathbf{x}_1\sim\mathbf{x}_6$ (blue balls) are represented by a graph $\mathcal{G}$, and the relationship between pairs of examples are modeled by the edges between them. Suppose there are totally $M$ learners (\emph{i.e.} base classifiers) constituted by $M$ different propagation models, in which the adjacency matrices \cite{Zhu03semisupervisedlearning,wang2016semi} for describing $\mathcal{G}$ are $\mathbf{W}^{(1)},\mathbf{W}^{(2)},\cdots,\mathbf{W}^{(M)}$, respectively. Then we introduce $M$ teachers to ``teach'' these $M$ learners, and each teacher is responsible for teaching one learner. In each propagation round (see Fig.~\ref{fig:framework}(b)), the $m$-th ($m=1,2,\cdots,M$) teacher generates the curriculum $\mathcal{S}^{(m)}$ that is the simplest to the $m$-th learner, and then the overall simplest curriculum set $\mathcal{S}^{*}$ is established by comprehensively considering all the curriculums $\mathcal{S}^{(1)}\sim\mathcal{S}^{(M)}$ recommended by different teachers. 
Therefore, $\mathcal{S}^{*}$ is established via an ensemble teaching way. Given $\mathcal{S}^{*}$, every learner will ``learn'' these simplest examples (see Fig.~\ref{fig:framework}(c)) by propagating the labels from $\mathcal{L}$ to $\mathcal{S}^{*}$, and the obtained label matrix is $\mathbf{F}^{(m)}\in\mathbb{R}^{n\times c}$ ($m=1,2,\cdots,M$, and $c$ is the total number of classes). After that, an integrated label matrix $\mathbf{F}$ is computed as the sum of $\mathbf{F}^{(1)},\mathbf{F}^{(2)},\cdots,\mathbf{F}^{(M)}$ weighted by $\bm{\omega}^{(1)},\bm{\omega}^{(2)},\cdots,\bm{\omega}^{(M)}$, respectively. We call this hybrid diffusion because multiple label propagation models are combined to achieve accurate learning. Finally, a learning feedback is delivered to the ensemble of teachers to assist them to correctly determine the subsequent simplest curriculum. Above teaching and learning process iterates until all the examples in $\mathcal{U}$ have been selected, and the produced label matrix is denoted by $\mathring{\mathbf{F}}$. The $(i,j)$-th element in $\mathring{\mathbf{F}}$ (or $\mathbf{F}^{(m)}$ and $\mathbf{F}$ mentioned above) encodes the probability of the $i$-th ($i=1,2,\cdots,n$) example $\mathbf{x}_i$ belonging to the $j$-th ($j=1,2,\cdots,c$) class $\mathcal{C}_{j}$.\par

The adopted hybrid label propagation combines the advantages possessed by different learners, and thus can improve the performance of every single classifier. Besides, the incorporated multiple teachers also cooperate with each other to wisely re-organize the propagation sequence so that the unlabeled examples are logically classified from simple to difficult. As a result, the proposed algorithm is able to outperform other existing typical methods, which will be empirically revealed in Section~\ref{sec:Experiments}. Furthermore, the process of learning from simple to difficult is also consistent with humans' cognitive process \cite{elman1993learning, khan2011humans}, who gradually gain rich and complex knowledge from the childish stage to the mature stage.\\
\textbf{Notations:} Throughout this paper, we use bold capital letter, bold lowercase letter and italic letter to represent matrix, vector and constant, respectively. The superscript, \emph{e.g.} ``$(m)$'', denotes that the related variable is associated with the $m$-th teacher or learner. Given a matrix $\mathbf{A}$, $\mathbf{A}_i$ is the $i$-the row of $\mathbf{A}$, and $\mathbf{A}_{ij}$ is $\mathbf{A}$'s $(i,j)$-th element. $tr(\mathbf{A})$ denotes the trace of $\mathbf{A}$ given $\mathbf{A}$ is square. The notations $\left\|\mathbf{A}\right\|_{\mathrm{F}}$ and $\left\|\mathbf{A}\right\|_{2,1}$ correspond to $\mathbf{A}$'s Frobenius norm and $l_{2,1}$ norm, which are defined by $\left\|\mathbf{A}\right\|_{\mathrm{F}}=\sqrt{\sum_{i,j}\mathbf{A}_{ij}^2}$ and $\left\|\mathbf{A}\right\|_{2,1}=\sum_i\sqrt{\sum_{j}\mathbf{A}_{ij}^2}$, respectively. Besides, we use ``$\circ$'' to represent the Hadamard product between two matrices, which means $(\mathbf{A}\circ\mathbf{B})_{ij}=\mathbf{A}_{ij}\mathbf{B}_{ij}$.

\section{Related Work}
\label{sec:related work}

Our work is related to semi-supervised label propagation, ensemble learning and machine teaching, therefore in this section we review some representative literatures on these three topics.

\subsection{Semi-supervised Label Propagation}
Label propagation belongs to the scope of semi-supervised learning \cite{Goldberg09introductionto}, of which the target is to classify a massive number of unlabeled examples given the existence of only a few labeled examples. Most existing semi-supervised algorithms are based on Support Vector Machines (SVMs) \cite{joachims1999transductive,li2011towards} or similarity graph model \cite{wang2009linear,Zhu03semisupervisedlearning,Zhou03learningwith,gong2015deformed,belkin2006manifold,Gong2015Fick}, and label propagation falls into the latter category.\par

Given an established similarity graph, a label propagation algorithm gradually propagates the labels of seed nodes (examples) to the unlabeled nodes. 
In each propagation round, the labels of all the examples are updated by considering both their previous states and the influence of other examples. Then the final steady state conveys the accurate labels of originally unlabeled examples. For example, Zhu \emph{et al.} \cite{Zhu2002Report} propose to iteratively propagate labels on a weighted graph by executing random walks with clamping operations. 
The probability of each unlabeled node to be absorbed by the different labeled nodes is employed to infer the labels of unlabeled examples.
Different from \cite{Zhu2002Report} which works
on asymmetric graph Laplacian, Zhou \emph{et al.} \cite{Zhou03learningwith} utilize a symmetric graph Laplacian to implement propagation and achieve satisfactory performance.
Besides, Wang \emph{et al.} \cite{wang2009linear} develop the linear neighbourhood propagation which assumes that any node in the graph can be linearly reconstructed by its $K$ nearest neighbours. In contrast to the above methods that are based on a static graph throughout the entire propagation, Wang \emph{et al.} \cite{wang2013dynamic} develop dynamic label propagation to adaptively update the edge weights so that the graph can always faithfully reflect the similarities between examples during propagation. Recently, Gong \emph{et al.} \cite{Gong2015Fick} treat label propagation on a graph as fluid diffusion on a plane, and successfully apply Fick's law of diffusion in physical area to facilitate propagation. 

\subsection{Ensemble Learning}
An ensemble learning method \cite{dietterich2000ensemble,zhou2012ensemble} incorporates multiple learning algorithms to obtain better performance than that could be obtained from any of the constituent learning algorithms. In other words, an ensemble of classifiers is a set of different classifiers of which the individual predictions are combined in some weighted or unweighted way to form an enhanced performance. The classifiers constituting the ensemble should comply with two criteria: 1) diversity, namely each base classifier of the ensemble should hold unique information and cover one side of the entire underlying facts; and 2) independence, namely every base classifier is able to draw a conclusion that is not so bad all by itself.\par

Roughly speaking, there are three representative fashions for building an ensemble learning algorithm, \emph{i.e.} boosting, bagging, and hybrid. Boosting sequentially trains a set of weak classifiers and then combines all of them to form a strong classifier, during which the later trained classifiers focus more on the errors made by the earlier classifiers. Perhaps the most well-known boosting algorithm is Adaboost \cite{freund1995desicion}. AdaBoost is adaptive because the examples are automatically re-weighted after each iteration so that the subsequent weak classifiers will pay more attention to the examples that are probably mislabeled by the previous classifiers. Different from boosting which trains weak classifiers sequentially, bagging generates the weak classifiers in parallel, and then aggregates them into a strong classifier. Random Forest (RF) \cite{breiman2001random} is a representative bagging method, in which every base classifier is a decision tree that accomplishes random feature selection. In boosting and bagging, the base classifiers are usually the same. In contrast, a hybrid method contains different base classifiers, and their outputs are properly fused to yield the final result. The proposed HyDEnT is a hybrid method.\par

Regarding ensemble label propagation, Woo \emph{et al.} \cite{woo2012semi} randomly sample several unlabeled sets, and deploy the method in \cite{wang2009linear} to build multiple base classifiers. Pan \emph{et al.} \cite{panensemble2015} propose to use ``Consensus Maximization Fusion Model'' to combine the label prediction results of multiple random walk classifiers. Lin \emph{et al.} \cite{lin2014ensemble} develop an ensemble propagation method for different labeled and unlabeled sources, in which local and global consensus are particularly considered. However, neither of above algorithms is hybrid nor has a teaching component as in our method.

\subsection{Machine Teaching}

The early works regarding machine teaching mainly focus on the the theory of ``teaching dimension'' 
\cite{shinohara1991teachability,balbach2006teaching}. Currently, there are mainly two trends in developing machine teaching algorithms. One trend assumes that the teacher knows the real labels of curriculum examples 
\cite{singla2014near,zhu2013machine,patil2014optimal,zhu2015machineteachingQA,mei2015using}, while in the other trend the teacher only knows the difficulty of unlabeled examples without accessing their real labels. Curriculum learning \cite{bengio2009curriculum} and self-paced learning \cite{kumar2010self,jiang2014self} belong to the second trend which argue that the learning process should follow the simple-to-difficult sequence.\par

Recently, Gong \emph{et al.} \cite{Gong2016TLLT,Gong2016AAAI} proposed the framework of Teaching-to-Learn and Learning-to-Teach (TLLT), which extends curriculum learning by adding a learning feedback to help teacher adaptively update the curriculums. Inspired by \cite{Gong2016TLLT}, this work attempts to adapt the TLLT framework to hybrid label propagation.


\section{Curriculum Generation via Ensemble Teaching}
\label{sec:CurriculumGenerationViaEnsembleTeaching}
This section introduces our ensemble teaching strategy. The ultimate target of ensemble teaching is to select the simplest curriculum examples for the learners in every propagation round. To this end, two factors should be considered by each of the $M$ teachers: firstly, the curriculum examples in $\mathcal{S}^{(m)}$ ($m=1,2,\cdots, M$) decided by the $m$-th teacher should be simple in terms of its associated learner; and secondly, these $M$ teachers are not isolated and they should share some common knowledge in determining the curriculum. Therefore, the ensemble teaching model can be formulated as the following optimization problem:
\begin{equation}\label{eq:1}
  \min \ \sum\nolimits_{m=1}^{M}{\mathcal{A}({{\mathcal{S}}^{(m)}})}+{{\beta }_{0}}\Omega ({\mathcal{S}^{(1)}},{\mathcal{S}^{(2)}},\cdots ,{\mathcal{S}^{(M)}}),
\end{equation}
where $\mathcal{A}(\cdot)$ is a function for the $m$-th teacher to select the simplest examples from its own viewpoint, $\Omega(\cdot)$ is a function modeling the relationship among different teachers, and $\beta_0>0$ is a trade-off parameter. Next we will explain the detailed expressions of $\mathcal{A}(\cdot)$ and $\Omega(\cdot)$.

\vspace{-5pt}
\subsection{Establishment of $\mathcal{A}(\cdot)$}
The function $\mathcal{A}(\cdot)$ is exploited to assist the individual $m$-th teacher to choose a curriculum set $\mathcal{S}^{(m)}\subset\mathcal{U}$. In this subsection, we temporarily drop the superscript $(m)$ of the appeared variables for simplicity. Suppose the adjacency matrix of graph $\mathcal{G}$ for the $m$-th teacher-learner pair is $\mathbf{W}$\footnote{Different propagation methods may have different ways for generating adjacency matrix. For example, \cite{Zhu2002Report} adopts Gaussian kernel function, \cite{Gong2015Fick} leverages Fick's law of diffusion, and \cite{wang2009linear} is based on the locally linear reconstruction.},
then the graph Laplacian is $\mathbf{L}=\mathbf{D}-\mathbf{W}$ where $\mathbf{D}$ is the degree matrix with diagonal elements computed by ${{\mathbf{D}}_{ii}}=\sum\nolimits_{j=1}^{n}{{\mathbf{W}}_{ij}}$. Based on $\mathbf{W}$, the $m$-th teacher will evaluate the difficulty of $\mathbf{x}_i\in\mathcal{U}$ from $\mathbf{x}_i$'s reliability and discriminability \cite{Gong2016TLLT}.
\subsubsection{Reliability}
To assess the propagation reliability of $\mathbf{x}_i\in\mathcal{U}$, we assign a random variable $y_{i}$ to the example $\mathbf{x}_i$, and treat the propagations on $\mathcal{G}$ as a Gaussian process \cite{zhu2003semiGaussian}. Therefore, this Gaussian process is modeled as a multivariate Gaussian distribution over the random variables
$\mathbf{y}\!=\!(y_1,\!\cdots\!,y_n)^\top$, which has a concise form $\mathbf{y}\!\sim\! \mathcal{N}(\mathbf{0},\mathbf{\Sigma})$
with its covariance matrix being $\mathbf{\Sigma}\!=\!{{(\mathbf{L}\!+\!{\mathbf{I}}/{{{\kappa }^{2}}})}^{-1}}$. Here $\mathbf{I}$ denotes the identity matrix, and the parameter ${{\kappa }^{2}}$ controls the ``sharpness'' of the distribution which is fixed to 100 throughout this paper. Therefore, a curriculum $\mathcal{S}$ is reliable w.r.t. the labeled set $\mathcal{L}$ if the conditional entropy $H(\mathbf{y}_{\mathcal{S}}|\mathbf{y}_{\mathcal{L}})$ is small, where $\mathbf{y}_{\mathcal{S}}$ and $\mathbf{y}_{\mathcal{L}}$ denote the subvectors of $\mathbf{y}$ corresponding to $\mathcal{S}$ and $\mathcal{L}$, respectively. This is because small $H(\mathbf{y}_{\mathcal{S}}|\mathbf{y}_{\mathcal{L}})$ suggests that the curriculum set $\mathcal{S}$ shows no ``surprise'' to the labeled set $\mathcal{L}$.\par

Based on above consideration, we use the property of multivariate Gaussian \cite{bishop2006pattern} and select the most reliable curriculum by optimizing:
\begin{equation}\label{eq:2}
\begin{split}
&\min_{\mathcal{S}\subset\mathcal{U}} ~H( \mathbf{y}_{\mathcal{S}}|\mathbf{y}_{\mathcal{L}} )\\
\Leftrightarrow&\min_{\mathcal{S}\subset\mathcal{U}}~H( \mathbf{y}_{\mathcal{S}\cup\mathcal{L}} )   -H( \mathbf{y}_{\mathcal{L}} )\\
\Leftrightarrow&\min_{ \mathcal{S}\subset \mathcal{U} } \left( \frac{s+l}{2}\big( 1+\ln 2\pi \big)+\frac{1}{2}\ln \big|\mathbf{\Sigma}_{\mathcal{S}\cup\mathcal{L},\mathcal{S}\cup\mathcal{L}}\big| \right)  \\
&\quad\quad\quad~~ -\left( \frac{l}{2}\big( 1+\ln 2\pi \big)+\frac{1}{2}\ln \big|\mathbf{\Sigma}_{\mathcal{L},\mathcal{L}}\big| \right)  \\
\Leftrightarrow&\min_{ \mathcal{S}\subset \mathcal{U} } \frac{s}{2}\big( 1+\ln 2\pi \big) + \frac{1}{2}\ln \frac{\big|\mathbf{\Sigma}_{\mathcal{S}\cup\mathcal{L},\mathcal{S}\cup\mathcal{L}}\big|}{\big|\mathbf{\Sigma}_{\mathcal{L},\mathcal{L}}\big|}, \end{split}
\end{equation}
where $\mathbf{\Sigma}_{\mathcal{L},\mathcal{L}}$ and $\mathbf{\Sigma}_{\mathcal{S}\cup\mathcal{L},\mathcal{S}\cup\mathcal{L}}$ are submatrices of $\mathbf{\Sigma}$ associated with the corresponding subscripts. By further partitioning $\mathbf{\Sigma}_{\mathcal{S}\cup\mathcal{L},\mathcal{S}\cup\mathcal{L}}
=\left( \begin{smallmatrix}
   \mathbf{\Sigma}_{\mathcal{S},\mathcal{S}}  &  \mathbf{\Sigma}_{\mathcal{S},\mathcal{L}}  \\
   \mathbf{\Sigma}_{\mathcal{L},\mathcal{S}}  &  \mathbf{\Sigma}_{\mathcal{L},\mathcal{L}}  \\
\end{smallmatrix} \right)$ where $\mathbf{\Sigma}_{\mathcal{S},\mathcal{S}}$ is the submatrix of $\mathbf{\Sigma}$ corresponding to $\mathcal{S}$, we have
\begin{equation}
  \frac{\left| {{\mathbf{\Sigma }}_{\mathcal{S}\cup\mathcal{L}, \mathcal{S}\cup\mathcal{L}}} \right|}{\left| {{\mathbf{\Sigma }}_{\mathcal{L},\mathcal{L}}} \right|}
  \!=\!\frac{\left| {{\mathbf{\Sigma }}_{\mathcal{L},\mathcal{L}}} \right|\left| {{\mathbf{\Sigma }}_{\mathcal{S},\mathcal{S}}}\!-\!{{\mathbf{\Sigma }}_{\mathcal{S},\mathcal{L}}}\mathbf{\Sigma }_{\mathcal{L},\mathcal{L}}^{-1}{{\mathbf{\Sigma }}_{\mathcal{L},\mathcal{S}}} \right|}{\left| {{\mathbf{\Sigma }}_{\mathcal{L},\mathcal{L}}} \right|}
  \!=\!\left| {{\mathbf{\Sigma }}_{\mathcal{S}|\mathcal{L}}} \right|, \nonumber
\end{equation}
where ${{\mathbf{\Sigma}}_{\mathcal{S}|\mathcal{L}}}$ is the covariance matrix of the conditional distribution
$p(\mathbf{y}_{\mathcal{S}}|\mathbf{y}_{\mathcal{L}})$ and is naturally positive semidefinite.
Therefore, the problem \eqref{eq:2} is equivalent to
\begin{equation}\label{eq:3}
\min_{\mathcal{S}\subseteq \mathcal{U}} ~\mathrm{tr}\big( \mathbf{\Sigma}_{\mathcal{S},\mathcal{S}}-\mathbf{\Sigma}_{\mathcal{S},\mathcal{L}}\mathbf{\Sigma}_{\mathcal{L},\mathcal{L}}^{-1}
\mathbf{\Sigma}_{\mathcal{L},\mathcal{S}} \big).
\end{equation}
\subsubsection{Discriminability}
A curriculum $\mathcal{S}$ is discriminable if the included examples are significantly inclined to certain classes. The tendency of an example $\mathbf{x}_{i}$ belonging to a class $\mathcal{C}_j$ is modeled by the average commute time between $\mathbf{x}_{i}$ and all the examples in $\mathcal{C}_j$, which is formally represented by
\begin{equation}\label{eq:4}
\bar{T}(\mathbf{x}_i,\mathcal{C}_j)=\frac{1}{n_{\mathcal{C}_j}}\sum\nolimits_{\mathbf{x}_{i'}\in \mathcal{C}_j}
T(\mathbf{x}_i,\mathbf{x}_{i'}).
\end{equation}
In Eq.~\eqref{eq:4}, $n_{\mathcal{C}_j}$ denotes the number of examples in the class $\mathcal{C}_j$; $T(\mathbf{x}_i,\mathbf{x}_{i'})$ is the commute time between examples $\mathbf{x}_i$ and $\mathbf{x}_{i'}$ that is defined as \cite{qiu2007clustering}:
\begin{equation}\label{eq:CommuteTime}
T(\mathbf{x}_i,\mathbf{x}_{i'})=\sum\nolimits_{k=1}^n  h(\lambda_k) \big( u_{ki}-u_{ki'} \big)^2,
\end{equation}
where $0\!=\!\lambda_1\! \le\! \cdots \!\le\! \lambda_n$ are the eigenvalues of $\mathbf{L}$,
and $\mathbf{u}_1,\!\cdots\!,\mathbf{u}_n$ are the associated eigenvectors; $u_{ki}$ denotes the $i$-th element of $\mathbf{u}_k$; $h(\lambda_k)=1/\lambda_k$ if $\lambda_k\ne 0$ and $h(\lambda_k)=0$ otherwise.\par
Therefore, suppose $\mathcal{C}_1$ and $\mathcal{C}_2$ are the two closest classes to $\mathbf{x}_i\in\mathcal{U}$ measured by average commute time, then $\mathbf{x}_i$ is discriminable if the gap $g(\mathbf{x}_i)=\bar{T}(\mathbf{x}_i,\mathcal{C}_{2})-\bar{T}(\mathbf{x}_i,\mathcal{C}_{1})$ is large. That is, the simplest curriculum in view of discriminability is found by solving
\begin{equation}\label{eq:5}
\min_{ \mathcal{S}=\{\mathbf{x}_{i_k}\in~\mathcal{U}\}_{k=1}^s } \sum\nolimits_{k=1}^s 1/g(\mathbf{x}_{i_k}),
\end{equation}
where $s$ is the amount of examples in the set $\mathcal{S}$.\par
By putting Eqs.~\eqref{eq:3} and \eqref{eq:5} together, we arrive at the following optimization problem:
\begin{equation}\label{eq:6}
\min\nolimits_{ \mathcal{S}=\{\mathbf{x}_{i_k}\in~\mathcal{U}\}_{k=1}^s } \mathcal{A}(\mathcal{S}),
\end{equation}
where $\mathcal{A}(\mathcal{S})$ appeared in Eq.~\eqref{eq:1} is defined by
\begin{equation}\label{eq:7}
  \mathcal{A}(\mathcal{S}) = \mathrm{tr}\big( \mathbf{\Sigma}_{\mathcal{S},\mathcal{S}}-\mathbf{\Sigma}_{\mathcal{S},\mathcal{L}}\mathbf{\Sigma}_{\mathcal{L},\mathcal{L}}^{-1}
\mathbf{\Sigma}_{\mathcal{L},\mathcal{S}} \big) \!+\!\sum\nolimits_{k=1}^s 1/g(\mathbf{x}_{i_k}).
\end{equation}\par
However, Eq.~\eqref{eq:6} is symbolic and cannot be directly solved, so we provide a mathematically tractable model for Eq.~\eqref{eq:6}. In each propagation round, the seed labels will be diffused to the unlabeled examples that are direct neighbors (denoted by the neighbouring set $\mathcal{B}$) of $\mathcal{L}$ on graph $\mathcal{G}$, so we only need to choose $s$ distinct examples from $\mathcal{B}$.
Therefore, for the $m$-th teacher, we introduce a binary selection matrix
$\mathbf{S}\in[0,1]^{b\times s}$ ($b$ is the size of $\mathcal{B}$) such that its $(i,j)$-th element $\mathbf{S}_{ij}$ represents the appropriateness of the $i$-th example in $\mathcal{B}$ being selected as the $j$-th element of the curriculum $\mathcal{S}$. Ideally, we hope $\mathbf{S}$ to have two properties: 1) every element $\mathbf{S}_{ij}$ is preferred to be $\{0,1\}$-binary, which indicates that the teacher strongly discourages or encourages $\mathbf{x}_i$ to be a curriculum example; and 2) $\mathbf{S}$ should be orthogonal, which ensures that every example is selected only once by the $m$-th teacher. Above two ideal properties can be mathematically achieved by optimizing $\min_{\mathbf{S}}~\left\|\mathbf{S}\circ\mathbf{S}-\mathbf{S}\right\|_{\mathrm{F}}^2\!+\!\left\|\mathbf{S}^{\top}\mathbf{S}\!-\!\mathbf{I}\right\|_{\mathrm{F}}^2$, where the first term and second term realize the properties 1) and 2), respectively. Besides, we introduce a diagonal matrix $\mathbf{G}\in\mathbb{R}^{b\times b}$ with the diagonal elements $\mathbf{G}_{ii}=1/g(\mathbf{x}_{i})$ for any $\mathbf{x}_{i}\in\mathcal{B}$, then the problem \eqref{eq:6} can be reformulated as follows:
\begin{equation}\label{eq:8}
\begin{split}
\min_{\mathbf{S}}& ~\mathrm{tr}\big( \mathbf{S}^{\top}\mathbf{\Sigma}_{\mathcal{B},\mathcal{B}}\mathbf{S}-
\mathbf{S}^{\top}\mathbf{\Sigma}_{\mathcal{B},\mathcal{L}} \mathbf{\Sigma}_{\mathcal{L},\mathcal{L}}^{-1} \mathbf{\Sigma}_{\mathcal{L},\mathcal{B}}\mathbf{S} \big)\\
&+\mathrm{tr}\big( \mathbf{S}^{\top} \mathbf{G}\mathbf{S} \big)+\beta_1\big(\left\|\mathbf{S}\circ\mathbf{S}-\mathbf{S}\right\|_{\mathrm{F}}^2
+\left\|\mathbf{S}^{\top}\mathbf{S}\!-\!\mathbf{I}\right\|_{\mathrm{F}}^2\big),\\
\end{split}
\end{equation}
where $\beta_1$ is a trade-off parameter. By further denoting
$\mathbf{R}=\mathbf{\Sigma}_{\mathcal{B},\mathcal{B}}-\mathbf{\Sigma}_{\mathcal{B},\mathcal{L}}\mathbf{\Sigma}_{\mathcal{L},\mathcal{L}}^{-1}
\mathbf{\Sigma}_{\mathcal{L},\mathcal{B}}+\mathbf{G}$, Eq.~\eqref{eq:8} is simplified as
\begin{flalign}\label{eq:9}
\min_{\mathbf{S}}~\mathrm{tr}\big( \mathbf{S}^{\top}\mathbf{R}\mathbf{S} \big)
+\beta_1\big(\left\|\mathbf{S}\circ\mathbf{S}-\mathbf{S}\right\|_{\mathrm{F}}^2
+\left\|\mathbf{S}^{\top}\mathbf{S}\!-\!\mathbf{I}\right\|_{\mathrm{F}}^2\big),
\end{flalign}
which is the curriculum generation model for a single teacher.

\subsection{Establishment of $\Omega(\cdot)$}
The term $\Omega(\cdot)$ exploits the relationship among $M$ different teachers so that they work collaboratively to render the optimal curriculum. In this work, we hope that all teachers can maximally draw a consensus on the determination of difficult unlabeled examples, and then the remaining examples are simple and should be included in the optimal curriculum $\mathcal{S}^{*}$. That is to say, we aim to find the solution of the following optimization problem:
\begin{equation}\label{eq:10}
\max\nolimits_{{\mathcal{S}^{(1)}},{\mathcal{S}^{(2)}},\cdots ,\mathcal{S}^{(M)} } \Omega ({\mathcal{S}^{(1)}},{\mathcal{S}^{(2)}},\cdots ,\mathcal{S}^{(M)}),
\end{equation}
where $\Omega (\cdot)$ is defined by
\begin{equation}\label{eq:11}
  \Omega ({\mathcal{S}^{(1)}},{\mathcal{S}^{(2)}},\cdots ,{\mathcal{S}^{(M)}})=\left|\bigcap_{m=1}^M\mathcal{U}-{\mathcal{S}^{(m)}}\right|
\end{equation}
with ``$\left|\cdot\right|$'' denoting the set size. The operation $\mathcal{U}-{\mathcal{S}^{(m)}}$ computes the complementary set of $\mathcal{S}^{(m)}$ in $\mathcal{U}$.\par

For realizing Eq.~\eqref{eq:10}, we put the selection matrices $\mathbf{S}^{(1)},\mathbf{S}^{(2)},\cdots,\mathbf{S}^{(M)}$ produced by the $M$ teachers together and obtain a stacked matrix $\bar{\mathbf{S}}=\left(\mathbf{S}^{(1)},\mathbf{S}^{(2)},\cdots,\mathbf{S}^{(M)}\right)$ with size $b\times(s\times M)$. As a result, we may use the $l_{2,1}$ norm on $\bar{\mathbf{S}}$ to discover the shared common knowledge across different teachers, then the difficult examples agreed by all the teachers can be found by solving
\begin{equation}\label{eq:12}
  \min\nolimits_{\bar{\mathbf{S}}}\left\|\bar{\mathbf{S}}\right\|_{2,1},
\end{equation}
and the indices of all-zero rows in the optimized $\bar{\mathbf{S}}$ correspond to the difficult examples agreed by all $M$ teachers that cannot be taken as curriculum.

\subsection{Optimization}
\label{sec:CurriculumGenerationViaEnsembleTeaching_Optimization}
By combining Eqs.~\eqref{eq:9} and \eqref{eq:12}, our ensemble teaching model is formally represented by
\begin{equation}\label{eq:14}
\begin{split}
&\min_{\bar{\mathbf{S}}}~Q(\bar{\mathbf{S}})=\sum\nolimits_{m=1}^M\mathrm{tr}\big( \mathbf{S}^{(m)\top}\mathbf{R}^{(m)}\mathbf{S}^{(m)} \big)\!+\!\beta_0\left\|\bar{\mathbf{S}}\right\|_{2,1} \\
&\quad+\!\beta_1\!\sum\nolimits_{m=1}^M\!\!\left(\left\|\mathbf{S}^{(m)}\!\circ\!\mathbf{S}^{(m)}\!-\!\mathbf{S}^{(m)}\right\|_\mathrm{F}^2
\!+ \!\left\|\mathbf{S}^{(m)\top}\mathbf{S}^{(m)}\!-\!\mathbf{I}\right\|_{\mathrm{F}}^2\right)\!\!,
\end{split}
\end{equation}
where $\beta_0>0$ is the trade-off parameter.\par

The problem \eqref{eq:14} can be easily solved via Block Coordinate Descent (BCD), which updates blocks of variables at every iteration until convergence. For our method, at one time we compute the gradient related to $\mathbf{S}^{(m)}$ ($m$ takes a value from $1,2,\cdots,M$), which is denoted by $\nabla_{\mathbf{S}^{(m)}} Q=\frac{\partial Q}{\partial \mathbf{S}^{(m)}}$, and then decrease the objective function $Q(\bar{\mathbf{S}})$ by updating $\mathbf{S}^{(m)}$ along the opposite direction of the gradient $\nabla_{\mathbf{S}^{(m)}} Q$. As such, the objective function can be gradually minimized by cyclically updating $\mathbf{S}^{(1)},\mathbf{S}^{(2)},\cdots,\mathbf{S}^{(M)}$.\par

Next we derive the updating rule for $\mathbf{S}^{(m)}$. According to the definition of $\left\|\cdot\right\|_{2,1}$ in the introduction, it is easy to see that $\left\|\bar{\mathbf{S}}\right\|_{2,1}=\mathrm{tr}(\bar{\mathbf{S}}^{\top}\mathbf{H}\bar{\mathbf{S}})$ where $\mathbf{H}$ is a diagonal matrix with the diagonal elements $\mathbf{H}_{ii}=\frac{1}{2\left\|\bar{\mathbf{S}}_i\right\|_2}$ ($\bar{\mathbf{S}}_i$ represents the $i$-th row of the matrix $\bar{\mathbf{S}}$). Practically, $\bar{\mathbf{S}}_i$ could be zero, so we slightly modify the strict definition of $\mathbf{H}_{ii}$ as $\mathbf{H}_{ii}=\frac{1}{2\left\|\bar{\mathbf{S}}_i\right\|_2+\zeta}$ with $\zeta$ being a very small positive number. Therefore, by recalling that $\bar{\mathbf{S}}=(\mathbf{S}^{(1)},\mathbf{S}^{(2)},\cdots,\mathbf{S}^{(M)})$, we know $\left\|\bar{\mathbf{S}}\right\|_{2,1}=\sum_{m=1}^M~\mathrm{tr}(\mathbf{S}^{(m)\top}\mathbf{H}\mathbf{S}^{(m)})$. Consequently, the subproblem related to $\mathbf{S}^{(m)}$ is
\begin{equation}\label{eq:15}
\begin{split}
&\min_{\mathbf{S}^{(m)}}~Q(\mathbf{S}^{(m)})\!=\!\mathrm{tr}\big( \mathbf{S}^{(m)\top}\!\mathbf{R}^{(m)}\mathbf{S}^{(m)} \big)\!+\!\beta_0\mathrm{tr}\big(\mathbf{S}^{(m)\top}\!\mathbf{H}\mathbf{S}^{(m)}\big) \\
&\qquad+\!\beta_1\!\left(\left\|\mathbf{S}^{(m)}\circ\mathbf{S}^{(m)}\!-\!\mathbf{S}^{(m)}\right\|_\mathrm{F}^2
\!+ \!\left\|\mathbf{S}^{(m)\top}\mathbf{S}^{(m)}\!-\!\mathbf{I}\right\|_{\mathrm{F}}^2\right).
\end{split}
\end{equation}\par
To obtain the updating rule for $\mathbf{S}^{(m)}$, we need to compute the gradient $\nabla_{\mathbf{S}^{(m)}}Q$. 
We first present a useful lemma:
\begin{lemma}\label{lemma1}
Given an $n_1\times n_2$ matrix $\mathbf{A}$, the derivative of $\left\|\mathbf{A}\circ\mathbf{A}-\mathbf{A}\right\|_{\mathrm{F}}^2$ w.r.t. $\mathbf{A}$ is $\frac{d\left\|\mathbf{A}\circ\mathbf{A}-\mathbf{A}\right\|_{\mathrm{F}}^2}{d\mathbf{A}}=2(\mathbf{A}\circ\mathbf{A}-\mathbf{A})\circ(2\mathbf{A}-\mathbf{E})$
where $\mathbf{E}$ is an all-one matrix of size $n_1\times n_2$.
\end{lemma}
\begin{proof}
Given the $(i,j)$-th element of $\mathbf{A}$ as $\mathbf{A}_{ij}$, then $\left\|\mathbf{A}\circ\mathbf{A}-\mathbf{A}\right\|_{\mathrm{F}}^2=\sum\nolimits_{i=1}^{{{n}_{1}}}{\sum\nolimits_{j=1}^{{{n}_{2}}}{{{\left( \mathbf{A}_{ij}^{2}-{{\mathbf{A}}_{ij}}\right)}^{2}}}}$. Therefore, $\frac{d\left\|\mathbf{A}\circ\mathbf{A}-\mathbf{A}\right\|_{\mathrm{F}}^2}{d\mathbf{A}_{ij}}=2(\mathbf{A}_{ij}^2-\mathbf{A}_{ij})(2\mathbf{A}_{ij}-1)$. Consequently, $\frac{d\left\|\mathbf{A}\circ\mathbf{A}-\mathbf{A}\right\|_{\mathrm{F}}^2}{d\mathbf{A}}=\frac{d\left\|\mathbf{A}\circ\mathbf{A}-\mathbf{A}\right\|_{\mathrm{F}}^2}{d\mathbf{A}_{ij}}\big|_{i=1\sim n_1, j=1\sim n_2}=2(\mathbf{A}\circ\mathbf{A}-\mathbf{A})\circ(2\mathbf{A}-\mathbf{E})$, which completes the proof.
\end{proof}\par
Based on Lemma~\ref{lemma1}, the gradient $\nabla_{\mathbf{S}^{(m)}}Q$ is derived as
\begin{equation}\label{eq:16}
\begin{split}
  &\nabla_{\mathbf{S}^{(m)}}Q\\
  &=2\left\{\mathbf{R}^{(m)}\mathbf{S}^{(m)}\!+\!\beta_0\mathbf{HS}^{(m)}\!+\!\beta_1\left[(\mathbf{S}^{(m)}\!\circ\!\mathbf{S}^{(m)}\!-\!\mathbf{S}^{(m)})\right.\right.\\
  &\qquad\qquad\quad\left.\left.\circ(2\mathbf{S}^{(m)}-\mathbf{E})
  +2\mathbf{S}^{(m)}(\mathbf{S}^{(m)\top}\mathbf{S}^{(m)}-\mathbf{I})\right] \right\}\\
  &=2\left\{\left[\mathbf{R}^{(m)}+\beta_0\mathbf{H}+\beta_1(2\mathbf{S}^{(m)}\mathbf{S}^{(m)\top}-\mathbf{I})\right]\mathbf{S}^{(m)}\right.\\
  &\qquad \qquad  \left.+\beta_1\!\left[2(\mathbf{S}^{(m)}\!\circ\!\mathbf{S}^{(m)}\!\circ\!\mathbf{S}^{(m)})\!-\!3(\mathbf{S}^{(m)}\circ\mathbf{S}^{(m)})\right]\right\}.
\end{split}
\end{equation}
As a result, $\mathbf{S}^{(m)}$ is updated by
\begin{equation}\label{eq:17}
  \mathbf{S}^{(m)}:=\mathbf{S}^{(m)}-\tau\nabla_{\mathbf{S}^{(m)}}Q,
\end{equation}
where $\tau$ is the stepsize satisfying the Wolfe line-search conditions \cite{wolfe1969convergence}. Note that in our algorithm, the updating of $\mathbf{S}^{(1)}, \mathbf{S}^{(2)},\cdots,\mathbf{S}^{(M)}$ are not correlated, so their updating can be efficiently conducted in parallel.\par
The entire BCD process for solving \eqref{eq:14} is presented in Algorithm~\ref{alg1}, and its convergence analysis is deferred to Section~\ref{sec:ProofOfConvergence}. Suppose the solution of Eq.~\eqref{eq:14} is $\bar{\mathbf{S}}^{*}$, then we force very small elements in $\bar{\mathbf{S}}^{*}$, \emph{e.g.} less than 0.001, to 0, and keep the other elements as they are. Therefore, the sparseness of the $i$-th row in $\bar{\mathbf{S}}^{*}$ indicates the overall difficulty of the $i$-th example in $\mathcal{B}$ evaluated by all the teachers. As a result, the candidate examples corresponding to the $s$ most non-sparse rows in $\bar{\mathbf{S}}^{*}$ are selected as the curriculum examples for current propagation. Moreover, by partitioning $\bar{\mathbf{S}}^{*}$ as $\bar{\mathbf{S}}^{*}=\left(\mathbf{S}^{(1)*},\mathbf{S}^{(2)*},\cdots,\mathbf{S}^{(M)*}\right)$ in which the $m$-th ($m=1,2,\cdots,M$) block corresponds to the optimal decision made by the $m$-th teacher, the value of ${\left( \sum\nolimits_{j}{\mathop{\mathbf{S}}_{ij}^{(m)*}} \right)}/{\left( \sum\nolimits_{j}{\mathop{\mathbf{S}}_{ij}^{*}} \right)}$ reflects the tendency of the $m$-th teacher to choose the $i$-th example in $\mathcal{B}$ as a curriculum example.

\begin{algorithm}[t]
\small
   \caption{BCD method for optimizing \eqref{eq:14}}
   \label{alg1}
\begin{algorithmic}[1]
   \STATE {\bfseries Input:} $\mathbf{R}^{(m)}$, $\beta_0$, $\beta_1$, $\varepsilon=10^{-4}$, $iter\_max=300$, initial $\mathbf{S}^{(m)}$ with $0\leq\mathbf{S}^{(m)}_{ij}\leq 1$, $iter=0$
   \REPEAT
   \STATE Compute $\bar{\mathbf{S}}$ as $\bar{\mathbf{S}}=\left(\mathbf{S}^{(1)},\mathbf{S}^{(2)},\cdots,\mathbf{S}^{(M)}\right)$
   \STATE Compute $\mathbf{H}$ as $\mathbf{H}_{ii}=\frac{1}{2\left\|\bar{\mathbf{s}}_i\right\|_2+\zeta}$
   \STATE // Update $\mathbf{S}^{(m)}$ ($m=1,2,\cdots,M$) in parallel 
   \FOR{$m=1:M$ }
   \STATE Compute gradient $\nabla_{\mathbf{S}^{(m)}}Q$ via Eq.~\eqref{eq:16}
   \STATE Decide the stepsize $\tau$ via Wolfe line-search \cite{wolfe1969convergence}
   \STATE $\mathbf{S}^{(m)}:=\mathbf{S}^{(m)}-\tau\nabla_{\mathbf{S}^{(m)}}Q$
   \ENDFOR
   \STATE $iter:=iter+1$
   \STATE // Check termination condition
   \UNTIL {$\left\| \bar{\mathbf{S}}^{(iter)}-\bar{\mathbf{S}}^{(iter-1)} \right\|_{\mathrm{F}}<\varepsilon$ or $iter=iter\_max$}
   \STATE {\bfseries Output:} the optimal solution $\bar{\mathbf{S}}^{*}$
\end{algorithmic}
\end{algorithm}

\section{Hybrid Label Propagation and Learning Feedback}
\label{sec:HybridLabelPropagationAndFeedback}
Given the simplest curriculum set $\mathcal{S}^{*}\!=\!\left\{\mathbf{x}_{1}^{*},\mathbf{x}_{2}^{*},\!\cdots\!,\mathbf{x}_{s}^{*}\right\}$ designated by the ensemble of teachers, the $M$ learners should ``learn'' these simple examples by propagating the labels from $\mathcal{L}$ to $\mathcal{S}^{*}$, and the obtained label matrices are $\mathbf{F}^{(1)},\cdots,\mathbf{F}^{(M)}$, respectively. Finally, $\mathbf{F}^{(1)},\cdots,\mathbf{F}^{(M)}$ are integrated into a consistent output $\mathbf{F}$ with the $i$-th row being the label vector of the $i$-th example, which is computed by
\begin{equation}\label{eq:18}
  \mathbf{F}_i=\sum\nolimits_{m=1}^{M}\omega_{i}^{(m)}\mathbf{F}_{i}^{(m)},
\end{equation}
and $\omega_{i}^{(m)}={\left( \sum\nolimits_{j}{\mathop{\mathbf{S}}_{ij}^{(m)*}} \right)}/{\left( \sum\nolimits_{j}{\mathop{\mathbf{S}}_{ij}^{*}} \right)}$. Note that the weight $\omega_{i}^{(m)}$ is equivalent to the tendency of the $m$-th teacher to choose the $i$-th example in $\mathcal{B}$ as a curriculum example. As such, a large weight is imposed on the label vector $\mathbf{F}_{i}^{(m)}$ of the $m$-th learner in generating $\mathbf{F}_i$ if the $m$-th teacher strongly recommends $\mathbf{x}_i$ as a curriculum example. This is because the strong recommendation from the $m$-th teacher indicates that it considers the examples $\mathbf{x}_i$ is quite simple for the $m$-th learner, therefore the learning result $\mathbf{F}_{i}^{(m)}$ is trustable and should be emphasized. In the $t$-th propagation round, the label matrices $\mathbf{F}_{i}^{(m)}$ ($m\!=\!1,2,\cdots,M$) are
\begin{equation}\label{eq:19}
  \mathbf{F}_{i}^{(m)[t]}\!=\!\left\{ \begin{split}
  & {{\mathbf{P}_{i}^{(m)}}}{{\mathbf{F}}^{[t-1]}},\ {{\mathbf{x}}_{i}}\in {{\mathcal{S}}^{*[1:t-1]}}\cup {{\mathcal{S}}^{*[t]}} \\
 & \mathbf{F}_{i}^{[0]} \  , \ \ \ \ \ \ \quad {{\mathbf{x}}_{i}}\in {{\mathcal{L}}^{[0]}}\cup ({{\mathcal{U}}^{[0]}}-{{\mathcal{S}}^{*[1:t]}}) \\
\end{split} \right.,
\end{equation}
where $\mathcal{S}^{*[1:t]}=\mathcal{S}^{*[1]}\cup\cdots\cup\mathcal{S}^{*[t]}$ with the superscript $[t]$ representing the $t$-th propagation round, and the iteration matrix $\mathbf{P}^{(m)}\!=\!{{\mathbf{D}}^{(m)-1}}\mathbf{W}^{(m)}$ is related to the specific base classifier (\emph{i.e.} learner). 
Eq.~\eqref{eq:19} suggests that the labels of the $t$-th curriculum and previously ``learned'' examples will change during the $t$-th propagation, whereas the labels of the initially labeled examples in ${{\mathcal{L}}^{[0]}}$ and the unclassified unlabeled examples in ${{\mathcal{U}}^{[0]}}\!-\!{{\mathcal{S}}^{*[1:t]}}$ are kept unchanged. The initial state for ${{\mathbf{x}}_{i}}$'s label vector $\mathbf{F}_{i}^{[0]}$ is
\begin{equation}\label{eq:20}
\!\!\mathbf{F}_{i}^{[0]}\!\!:=\!\left\{ \begin{split}
  &\!\!\! \underbrace{\left( {1}/{c},\cdots ,{1}/{c} \right)}_{c},\ \ \ \ \ \ \ \ \ \ \ \ \ \ \ \ \ \ \  {{\mathbf{x}}_{i}}\in {{\mathcal{U}}^{[0]}} \\
 & \!\!\!\left( 0,\cdots ,\underset{\underset{j-th\ element}{\mathop{\downarrow }}}{\mathop{1}},\cdots ,0 \right)\!\!,\ \ {{\mathbf{x}}_{i}}\!\in\! {{\mathcal{C}}_{j}}\!\in\! \mathcal{L}^{(0)} \\
\end{split} \right., \!\!\!
\end{equation}
where $c$ is the number of classes as defined in the introduction. The Eqs.~\eqref{eq:19} and \eqref{eq:20} together maintain the probability interpretation $\sum_{j=1}^c \mathbf{F}_{ij}^{[t]}=1$ for
any example $\mathbf{x}_i$ and all $t$-th ($t=0,1,2,\cdots$) propagation rounds.\par
When the $t$-th propagation is finished, we require the learners to deliver a learning feedback to the ensemble of teachers so that the teachers can determine the proper $(t\!+\!1)$-th curriculum $\mathcal{S}^{*[t+1]}$. Intuitively, if the $t$-th learning result is satisfactory, the teachers may assign a ``heavier'' curriculum to the learners for the next propagation, otherwise they should relieve the burden for the learners by decreasing the amount of curriculum examples in $\mathcal{S}^{*[t+1]}$. To achieve this effect, here we use the feedback function designed in \cite{Gong2016TLLT}, which is
\begin{equation}\label{eq:21}
\begin{split}
  g({\mathbf{F}}^{[t]})&=\exp \left[ -\gamma\frac{1}{{{s}^{[t]}}}H( {\mathbf{F}^{[t]}}) \right]\\
  &=\exp \left[ \frac{\gamma}{s^{[t]}} \sum\nolimits_{i=1}^{s^{[t]}}\sum\nolimits_{j=1}^c \big( \mathbf{F}^{[t]} \big)_{ij}
  \log_c\big( \mathbf{F}^{[t]} \big)_{ij} \right],
\end{split}
\end{equation}
where $\gamma $ is the parameter controlling the learning rate. A small $\gamma $ leads to more examples incorporated into the curriculum $\mathcal{S}^{*[t+1]}$, so less propagation rounds are needed for the learners to classify all the unlabeled examples. However, learning too ``heavy'' curriculum at one time will make the learners more easily to make mistakes. Therefore, both learning speed and learning accuracy should be considered when choosing $\gamma$.\par
Based on Eq.~\eqref{eq:21}, the number of examples included in the $(t+1)$-th curriculum is:
\begin{equation}\label{eq:22}
  {{s}^{[t+1]}}=\left\lceil {{b}^{[t+1]}}\cdot g({{\mathbf{F}}^{[t]}}) \right\rceil,
\end{equation}
where ${{b}^{[t+1]}}$ is the size of set $\mathcal{B}^{[t+1]}$ in the $(t\!+\!1)$-th iteration, and $\left\lceil \cdot  \right\rceil$ rounds up the element to the nearest integer.\par

Above ensemble teaching and hybrid learning iterates until all the unlabeled examples are propagated (\emph{i.e.} $\mathcal{U}=\varnothing$), and the obtained label matrix is denoted as $\mathring{\mathbf{F}}$. Similar to \cite{Gong2016TLLT}, starting from $\mathring{\mathbf{F}}$, we adopt the following Eq.~\eqref{eq:23} to drive the propagation process of every learner to the steady state:
\begin{flalign}\label{eq:23}
{{\mathring{\mathbf{F}}}^{*(m)}}=(1-\theta ){{(\mathbf{I}-\theta \mathbf{P}^{(m)})}^{-1}}\mathring{\mathbf{F}},
\end{flalign}
where the parameter $\theta$ is set to 0.05. Therefore, the final produced label matrix is ${\mathring{\mathbf{F}}}^{*}\!=\!\frac{1}{M}\sum_{m=1}^{M}\mathring{\mathbf{F}}^{*(m)}$.
As a consequence, the example ${{\mathbf{x}}_{i}}$ is classified into the $j$-th class, which satisfies $j\!=\!\arg\max_{j'\in\{1,\cdots,c\}} \mathring{\mathbf{F}}_{ij'}^*$. The complete HyDEnT algorithm is summarized in Algorithm~\ref{alg2}.

\noindent \textbf{Discussion}: Although the model established in this paper focuses on ensemble learning, it can be easily adapted to handling multi-modal cases \cite{wang2012unsupervisedmetric,xie2013multi}. Specifically, we may associate each modality with a teacher and a learner, and then combine the propagation outputs of different modalities as explained in Section~\ref{sec:HybridLabelPropagationAndFeedback}. To achieve such combination, we may require all teachers make consistent decisions on determining the curriculum examples, which is very similar to the idea detailed in Section~\ref{sec:CurriculumGenerationViaEnsembleTeaching}.

\begin{algorithm}[t]
\small
   \caption{The summarization of HyDEnT algorithm}
   \label{alg2}
\begin{algorithmic}[1]
   \STATE {\bfseries Input:} $l$ labeled examples $\mathcal{L}\!=\!\left\{\mathbf{x}_{1},\!\cdots\!,\mathbf{x}_{l}\right\}$ with known labels $y_1,\!\cdots\!,y_l$; $u$ unlabeled examples $\mathcal{U}\!=\!\left\{\mathbf{x}_{l\!+\!1},\!\cdots\!,\mathbf{x}_{l\!+\!u}\right\}$ with unknown labels  $y_{l\!+\!1},\!\cdots\!,y_{l\!+\!u}$; Parameters $\beta_0$, $\beta_1$, $\gamma$;
   \STATE $//$ Pre-compute variables of each learner
   \STATE $\forall~ m\!=\!1,\!\cdots\!,M$, compute adjacency matrix $\mathbf{W}^{(m)}$, graph Laplacian $\mathbf{L}^{(m)}$, iteration matrix $\mathbf{P}^{(m)}$, covariance matrix $\mathbf{\Sigma}^{(m)}$, and pairwise commute time $T(\mathbf{x}_i,\mathbf{x}_{i^{\prime}})$, $\forall \mathbf{x}_i,\mathbf{x}_{i^{\prime}}\in\mathcal{L}\cup\mathcal{U}$;
   \REPEAT
   \STATE $//$ Ensemble teaching
   \STATE Compute $\mathbf{R}^{(m)}$ appeared in Eq.~\eqref{eq:9};
   \STATE Generate optimal curriculum $\mathcal{S}^{*}$ agreed by all the teachers via solving Eq.~\eqref{eq:14} (Algorithm~\ref{alg1});
   \STATE $//$ hybrid label propagation
   \STATE For each learner, compute the label matrix $\mathbf{F}^{(m)}$ via Eq.~\eqref{eq:19};

   \STATE Compute the integrated label matrix $\mathbf{F}$ via Eq.~\eqref{eq:18};
   \STATE $//$ Establish learning feedback
   \STATE Compute the value of feedback function $g(\mathbf{F})$ via Eq.~\eqref{eq:21};
   \STATE Compute the size of $(t\!+\!1)$-th curriculum $s^{[t\!+\!1]}$ via Eq.~\eqref{eq:22};
   \STATE $//$ Update sets
   \STATE $\mathcal{L}\!:=\!\mathcal{L}\cup \mathcal{S}^{*}$; $\mathcal{U}\!:=\!\mathcal{U}\!-\!\mathcal{S}^{*}$;
   \UNTIL{$\mathcal{U}=\varnothing$ };
   \STATE Drive the propagations of all $M$ learners to the steady states via Eq.~\eqref{eq:23}, and the resultant label matrices are $\mathring{\mathbf{F}}^{*(1)},\cdots,\mathring{\mathbf{F}}^{*(M)}$;

   \STATE Compute the final label matrix by ${\mathring{\mathbf{F}}}^{*}\!=\!\frac{1}{M}\sum_{m=1}^{M}\mathring{\mathbf{F}}^{*(m)}$;
  \STATE{Decide the label of original unlabeled example $\mathbf{x}_i$ as $y_i\!=\!\arg\max_{j\in\{1,\cdots,c\}} \mathring{\mathbf{F}}_{ij}^*$;}
   \STATE {\bfseries Output:} Class labels $y_{l\!+\!1},\cdots,y_{l\!+\!u}$;
\end{algorithmic}
\end{algorithm}

\section{Convergence Analyses}
\label{sec:ProofOfConvergence}
In this section, we discuss the convergence property of the BCD method in Algorithm~\ref{alg1}. Before proving the convergence of Algorithm~\ref{alg1}, we first present a useful lemma:\\
\begin{lemma}\label{lemma:Feiping}\cite{Nie2010Efficient}
Given any two vectors $\mathbf{a},\mathbf{b}\neq\mathbf{0}$, the following inequality holds:
\begin{equation}
  \left\|\mathbf{a}\right\|_2-\frac{\left\|\mathbf{a}\right\|_2^2}{2\left\|\mathbf{b}\right\|_2}
  \leq \left\|\mathbf{b}\right\|_2-\frac{\left\|\mathbf{b}\right\|_2^2}{2\left\|\mathbf{b}\right\|_2}.\nonumber
\end{equation}
\end{lemma}

Based on Lemma~\ref{lemma:Feiping}, we have the following theorem to guarantee the convergence of Algorithm~\ref{alg1}.
\begin{theorem}\label{thm:convergence}
Algorithm~\ref{alg1} monotonically decreases the value of objective function in \eqref{eq:14} until convergence.
\end{theorem}
\begin{proof}
To facilitate the proof, we decompose \eqref{eq:14} as $Q(\bar{\mathbf{S}})=Q_1(\bar{\mathbf{S}})+Q_2(\bar{\mathbf{S}})$, where
\begin{equation}\label{eq:24}
\begin{split}
Q_1(\bar{\mathbf{S}})&=\sum\nolimits_{m=1}^M\left[\mathrm{tr}\big( \mathbf{S}^{(m)\top}\mathbf{R}^{(m)}\mathbf{S}^{(m)} \big)\right. \\
&\left.+\beta_1\!\left(\!\left\|\mathbf{S}^{(m)}\!\circ\!\mathbf{S}^{(m)}\!-\!\mathbf{S}^{(m)}\right\|_\mathrm{F}^2
\!+\!\left\|\mathbf{S}^{(m)\top}\!\mathbf{S}^{(m)}\!-\!\mathbf{I}\right\|_{\mathrm{F}}^2\right)\right]
\end{split}
\end{equation}
and
\begin{equation}\label{eq:25}
Q_2(\bar{\mathbf{S}})=\beta_0\left\|\bar{\mathbf{S}}\right\|_{2,1}
=\beta_0\mathrm{tr}(\bar{\mathbf{S}}^{\top}\mathbf{H}\bar{\mathbf{S}}),
\end{equation}
where $\mathbf{H}$ is a diagonal matrix defined in Section~\ref{sec:CurriculumGenerationViaEnsembleTeaching_Optimization}.\par
Suppose that after one iteration, the variables $\bar{\mathbf{S}}$ is updated as $\bar{\mathbf{S}}_{new}$, then according to the definition of $\mathbf{H}$, we have
\begin{equation}\label{eq:26}
  Q_1(\bar{\mathbf{S}}_{new})+\beta_0{\sum\limits_{i=1}^{b}{\frac{\left\| {(\bar{\mathbf{S}}_{new})}_{i} \right\|_{2}^{2}}{2\left\| \bar{\mathbf{S}}_{i} \right\|_{2}}}}\leq Q_1(\bar{\mathbf{S}})+\beta_0{\sum\limits_{i=1}^{b}{\frac{\left\| {{\bar{\mathbf{S}}}_{i}} \right\|_{2}^{2}}{2\left\| \bar{\mathbf{S}}_{i} \right\|_{2}}}},
\end{equation}
where ${(\bar{\mathbf{S}}_{new})}_{i}$ represents the $i$-th row of matrix $\bar{\mathbf{S}}_{new}$. Besides, according to Lemma~\ref{lemma:Feiping}, for each $i$ we obtain
\begin{equation}\label{eq:27}
    {{\left\| {{(\bar{\mathbf{S}}_{new})}_{i}} \right\|}_{2}}-\frac{\left\| {{(\bar{\mathbf{S}}_{new})}_{i}} \right\|_{2}^{2}}{2{{\left\| \bar{\mathbf{S}}_{i} \right\|}_{2}}}\le {{\left\| \bar{\mathbf{S}}_{i} \right\|}_{2}}-\frac{\left\| \bar{\mathbf{S}}_{i}\right\|_{2}^{2}}{2{{\left\| \bar{\mathbf{S}}_{i} \right\|}_{2}}}.
\end{equation}
Therefore, it is straightforward to see
\begin{equation}\label{eq:28}
    \beta_0\!\sum_{i=1}^{b}\!\left[\!{{\left\| {{(\bar{\mathbf{S}}_{new})}_{i}} \right\|}_{2}}\!-\!\frac{\left\| {{(\bar{\mathbf{S}}_{new})}_{i}} \right\|_{2}^{2}}{2{{\left\| \bar{\mathbf{S}}_{i} \right\|}_{2}}}\right]
    \!\le\! \beta_0\!\sum_{i=1}^{b}\!\left[{{\left\| \bar{\mathbf{S}}_{i} \right\|}_{2}}\!-\!\frac{\left\| \bar{\mathbf{S}}_{i} \right\|_{2}^{2}}{2{{\left\| \bar{\mathbf{S}}_{i} \right\|}_{2}}}\!\right]\!\!.
\end{equation}
By adding Eq.~\eqref{eq:26} and Eq.~\eqref{eq:28}, we immediately have
\begin{equation}\label{eq:29}
  Q_1(\bar{\mathbf{S}}_{new})+\beta_0{\sum\nolimits_{i=1}^{b}{{\left\| {{(\bar{\mathbf{S}}_{new})}_{i}} \right\|_{2}}}}
  \leq Q_1(\bar{\mathbf{S}})+\beta_0{\sum\nolimits_{i=1}^{b}{{\left\| {\bar{\mathbf{S}}_{i}} \right\|_{2}}}}.
\end{equation}
By noticing $\sum\nolimits_{i=1}^{b}{{\left\| {(\bar{\mathbf{S}}_{new})_{i}} \right\|_{2}}}=\left\|\bar{\mathbf{S}}_{new}\right\|_{2,1}$ and $\sum\nolimits_{i=1}^{b}{{\left\| {\bar{\mathbf{S}}_{i}} \right\|_{2}}}=\left\|\bar{\mathbf{S}}\right\|_{2,1}$, it is obvious that Eq.~\eqref{eq:29} is equivalent to
\begin{equation}\label{eq:30}
  Q(\bar{\mathbf{S}}_{new}) \leq Q(\bar{\mathbf{S}}).
\end{equation}
Therefore, the objective function in Eq.~\eqref{eq:14} monotonically decreases during the iterations.\par
Besides, it can be easily verified that the objective function in Eq.~\eqref{eq:14} has the lower bound 0. Consequently, Algorithm~\ref{alg1} will finally converge and thus Theorem~\ref{thm:convergence} is proved.
\end{proof}

\section{Complexity Analyses}
\label{sec:ComplexityAnalyses}
To analyze the time complexity of Algorithm~\ref{alg2}, we first calculate the complexity of Algorithm~\ref{alg1}, as it is used in Line 7 of Algorithm~\ref{alg2}. In single propagation round, the complexities of Line~4 and Line~7 in Algorithm~\ref{alg1} are $\mathcal{O}(bsM)$ and $\mathcal{O}((b^{2}s+bs)M)$, respectively, where $b$ is the size of neighbouring set $\mathcal{B}$ and $s$ is the amount of curriculum examples in $\mathcal{S}^{*}$. However, both $b$ and $s$ vary in different propagation rounds, and they are not larger than $u$ ($u$ is the number of originally unlabeled examples), so the complexities of Line~4 and Line~7 are bounded by $\mathcal{O}(Mu^2)$ and $\mathcal{O}(Mu^3)$ for simplicity. Suppose Lines~2$\sim$12 are iterated $T_1$ times, then the complexity of Algorithm~\ref{alg1} is $\mathcal{O}(T_{1}Mu^3)$.\par

Regarding Algorithm~\ref{alg2}, the graph construction and computation of $\mathbf{\Sigma}$ for each teacher-learner pair in Line~3 take $\mathcal{O}(n^2)$ and $\mathcal{O}(n^3)$ complexities, respectively. To compute the commute time between all pairs of examples in Eq.~\eqref{eq:9}, one has to conduct eigen-decomposition on $\mathbf{L}^{(m)}$, of which the complexity is $\mathcal{O}(n^3)$. The complexity of Line 7 is bounded by $\mathcal{O}(T_{1}Mu^3)$ as explained above. It is also easy to find that the complexities of Lines~9, 12 and 17 are $\mathcal{O}(Munc)$, $\mathcal{O}(uc)$, and $\mathcal{O}(Mn^2)$, respectively, where the Eq.~\eqref{eq:23} in Line 17 is computed by transforming it into a group of linear equations with an $n\times n$ coefficient matrix. Therefore, taking all the above results into consideration and suppose that Lines~4$\sim$16 are repeated $T_2$ times, our HyDEnT algorithm takes $\mathcal{O}\big(Mn^3+(T_{1}Mu^3+Munc+uc)T_2\big)$ complexity. From above analyses, we see that the most computationally expensive steps in our algorithm lie in the calculations of covariance matrix $\mathbf{\Sigma}^{(m)}$ and commute time between all pairs of examples, of which the complexities are $\mathcal{O}(Mn^3)$ for all $M$ teachers. Fortunately, they can be pre-computed ahead of conducting the iterative propagation process. 
Therefore, the complexity of our HyDEnT is acceptable.

\section{Experimental Results}
\label{sec:Experiments}
In this section, we provide thorough experimental results to show the effectiveness of the proposed HyDEnT algorithm. We first demonstrate that both ensemble teaching and hybrid label propagation incorporated in our method are beneficial to obtaining the encouraging performance, and then compare HyDEnT with other state-of-the-art approaches on various classification tasks related to image, video, and audio.\par

In this paper, we utilize Harmonic Functions (HF) proposed in \cite{Zhu2002Report} and Fick's Law Assisted Propagation (FLAP) presented in \cite{Gong2015Fick} as two learners in HyDEnT, as they are state-of-the-art propagation algorithms developed so far. The iterative propagation rules in HF and FLAP are different. HF conducts label propagation on the adjacency matrix $\mathbf{W}$ with ${{\mathbf{W}}_{ij}}\!=\!\exp \left( -{\left\| {{\mathbf{x}}_{i}}\!-\!{{\mathbf{x}}_{j}} \right\|^{2}}/{(2{{\sigma}^{2}})} \right)$ ($\sigma $ is the Gaussian kernel width) if $\mathbf{x}_i$ and $\mathbf{x}_j$ are connected. Differently, FLAP relates $\mathbf{W}_{ij}$ to the diffusion distance $d_{ij}$ between $\mathbf{x}_i$ and $\mathbf{x}_j$, and also favors self-loop around each $\mathbf{x}_i$. We hope that these two learners complement to each other to yield good performance.

\subsection{Verification of Our Method}
\label{sec:VerificationOfOurMethod}

\begin{figure}
  \centering
  \includegraphics[width=\linewidth]{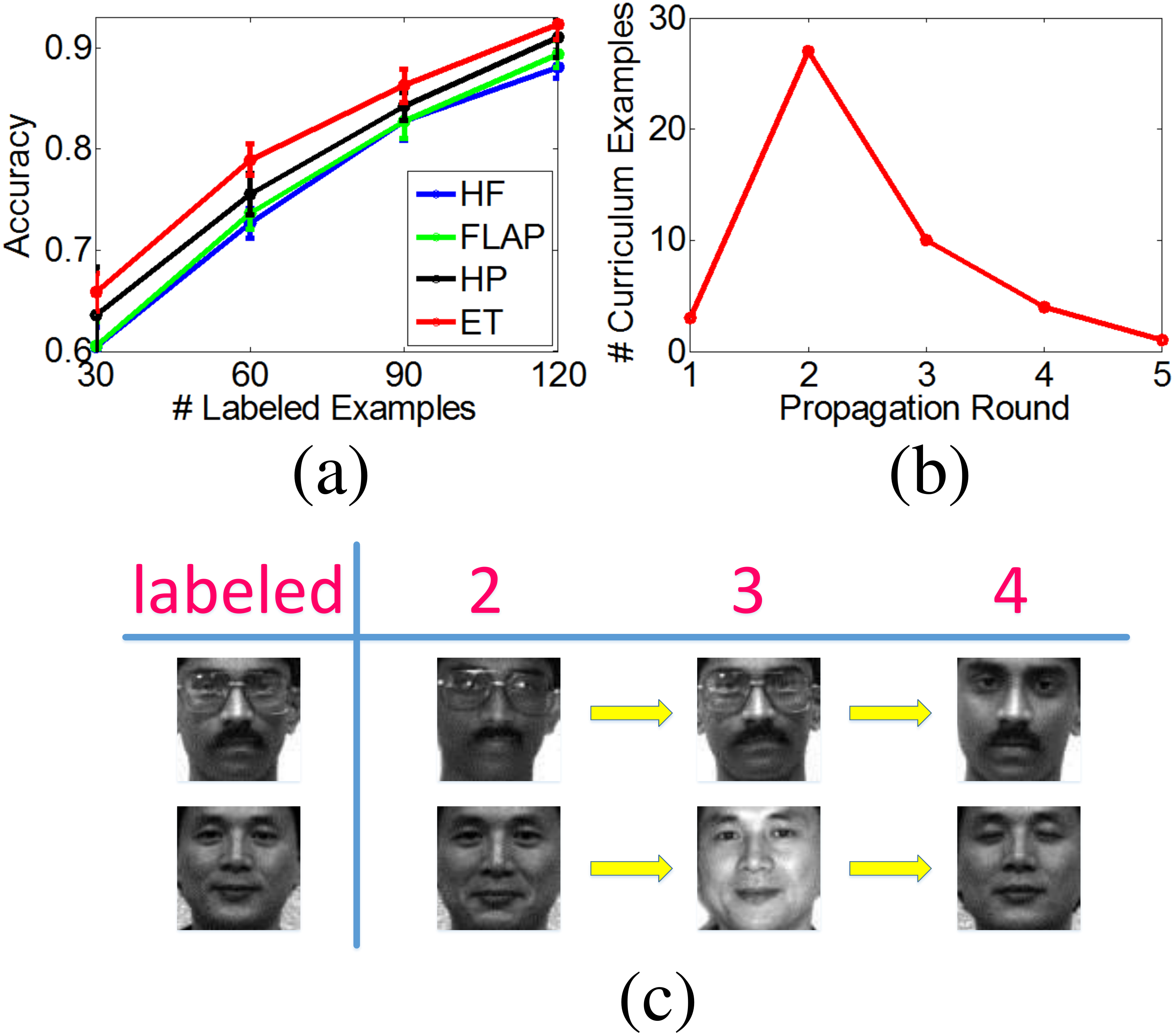}\\
  \caption{Verifications and illustrations of our developed HyDEnT algorithm. (a) compares the performances of two adopted learners (``HF'' and ``FLAP''), hybrid propagation (``HP''), and ensemble teaching (``ET''). (b) plots the number of selected curriculum examples in different propagation rounds. (c) presents the simplest face images of two individuals selected by teachers during the 2nd$\sim$4th propagation rounds.}\label{fig:AlgorithmVerification}
  \vspace{-10pt}
\end{figure}

As mentioned in the introduction, our algorithm contains two critical operations for boosting the performance, \emph{i.e.} ensemble teaching and hybrid label propagation. Here we use \emph{Yale}\footnote{\url{http://cvc.yale.edu/projects/yalefaces/yalefaces.html}} face recognition dataset to empirically demonstrate the usefulness of both operations.\par

The \emph{Yale} dataset contains 165 grayscale images of 15 individuals, and each individual constitutes a class. Every individual has 11 face images covering a variety of facial expressions and configurations such as normal, happy, wearing glasses, and so on. The resolution of every face image is 64$\times$64, so we directly rearrange each image to a 4096-dimensional long vector as input for our experiment.\par

We first present the results of two deployed learners HF \cite{Zhu2002Report} and FLAP \cite{Gong2015Fick}. Specifically, we investigate the classification accuracies under different numbers of labeled examples $l$, and for each $l$ the reported accuracy is averaged over the outputs of ten independent implementations. The splits of labeled set and unlabeled set are different in these ten implementations, however in one implementation such split is identical for all the compared settings. To show the effectiveness of Hybrid Propagation (denoted ``HP''), we average the generated label matrices of HF and FLAP in each propagation, and report the accuracy when the propagation process converges. To further show the improvements brought by Ensemble Teaching (dubbed ``ET''), we utilize the teaching algorithm introduced in Section~\ref{sec:CurriculumGenerationViaEnsembleTeaching} to select the simplest curriculum examples during each propagation round, so that the unlabeled examples are classified from simple to difficult in the entire propagation process. Note that the ET setting is actually the HyDEnT algorithm developed in this paper.\par

We build 5 nearest neighborhood (5-NN) graph for all the settings including HF, FLAP, HP, and ET, and the involved Gaussian kernel width $\sigma$ is set to 1. Besides, both $\beta_0$ and $\beta_1$ appeared in Eq.~\eqref{eq:14} are tuned to 100. The comparison results are presented in Fig.~\ref{fig:AlgorithmVerification}(a). It is observed that the two learners HF and FLAP perform comparably on this dataset with different choices of $l$ (see blue and green curves). However, when they are combined together and implemented in a hybrid way, the deficiency of each of them can be repaired and thus better results can be achieved (see black curve). Moreover, when the learners are ``taught'' by the ensemble of teachers, we notice that the classification accuracy can be further enhanced (see red curve). Above observations suggest that the teaching algorithm developed in Section~\ref{sec:CurriculumGenerationViaEnsembleTeaching} and hybrid propagation explained in Section~\ref{sec:HybridLabelPropagationAndFeedback} are helpful for obtaining satisfactory performance.\par

We also plot the number of selected curriculum examples (\emph{i.e.} $s$) in one independent implementation when $l=120$. From Fig.~\ref{fig:AlgorithmVerification}(b), we see that HyDEnT needs five propagation rounds to classify all the unlabeled examples, and most of the unlabeled examples are classified in the middle stage (\emph{e.g.} 2nd and 3rd propagations). The reason may be that at this stage the learners have gained the richest knowledge accumulated in the early stage. However, when the propagation goes to later stage, the difficulty of curriculums gradually increases, so the teachers become very ``conservative'' and assign less examples to the learners in one propagation. To demonstrate this point, we present some face images selected by the teachers during the 2nd$\sim$4th propagation rounds (see Fig.~\ref{fig:AlgorithmVerification}(c)). We see that the images chosen for the 2nd and 3rd propagations are very similar to the labeled images, and thus they can be easily and efficiently learned. In contrast, the curriculum face images for the 4th propagation look very different from the previous examples. The individual in the first row takes off his glasses, and the man in the second row closes his eyes. Such appearance variations make the examples difficult to learn, so their classifications are postponed by the teachers. As a result, Fig.~\ref{fig:AlgorithmVerification}(c) indicate that the unlabeled examples are generally learned via a simple-to-difficult order, which is consistent with our initial anticipation.\par

Finally, it can be noted that $\beta_0$ and $\beta_1$ in Eq.~\eqref{eq:14} are two critical tuning parameters in our proposed HyDEnT algorithm, so here we investigate whether the output of HyDEnT is sensitive to the variations of these two parameters. To be specific, we change one of $\beta_0$ and $\beta_1$ from $10^0$ to $10^3$ while keeping the other one fixed, and then examine the classification accuracies generated by HyDEnT. The results presented in Fig.~\ref{fig:ParaSensitivity} clearly indicate that HyDEnT is very robust to the variations of
these two parameters, so they can be easily tuned for practical use.

\begin{figure}
  \centering
  \includegraphics[width=\linewidth]{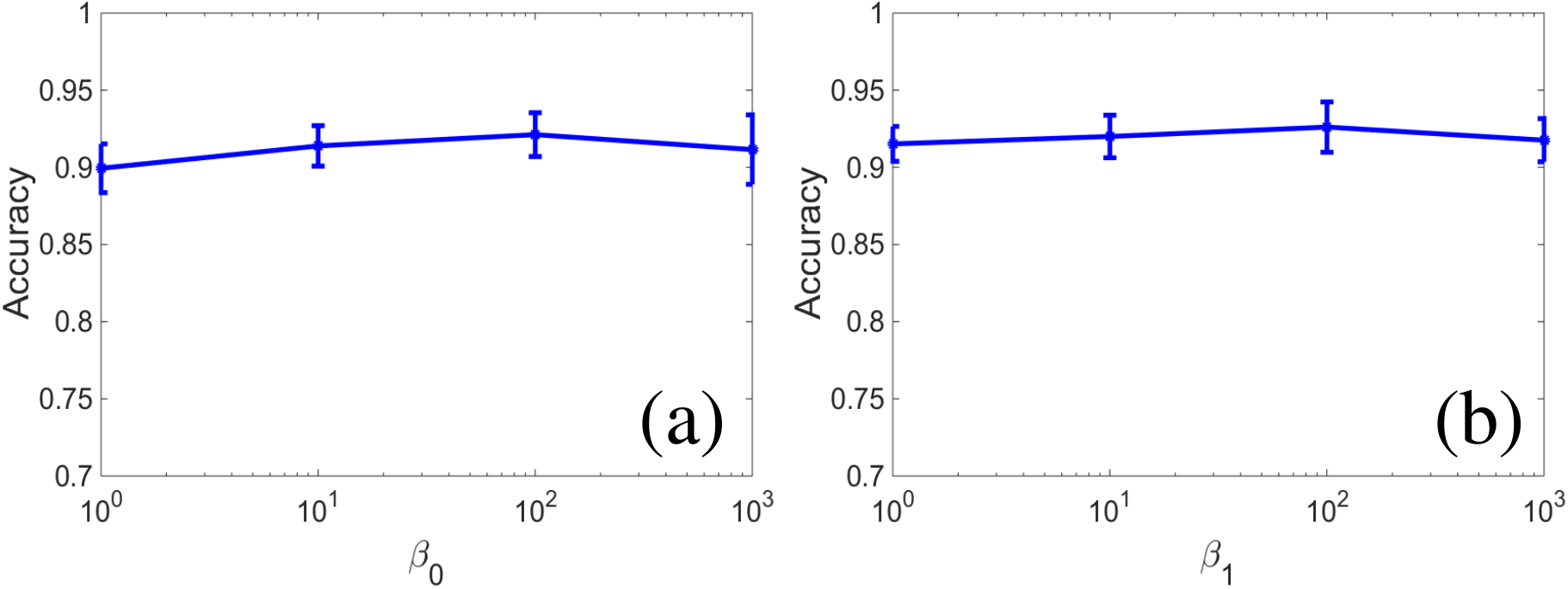}\\
  \caption{Parametric sensitivity of HyDEnT. (a) and (b) plot the classification accuracy w.r.t. the change of $\beta_0$ and $\beta_1$, respectively.}\label{fig:ParaSensitivity}
  \vspace{-12pt}
\end{figure}

\subsection{Single Teacher Vs. Multiple Teachers}
\label{sec:SingleVsMultipleTeacher}

\begin{figure*}[t]
  \centering
  \includegraphics[width=0.9\linewidth]{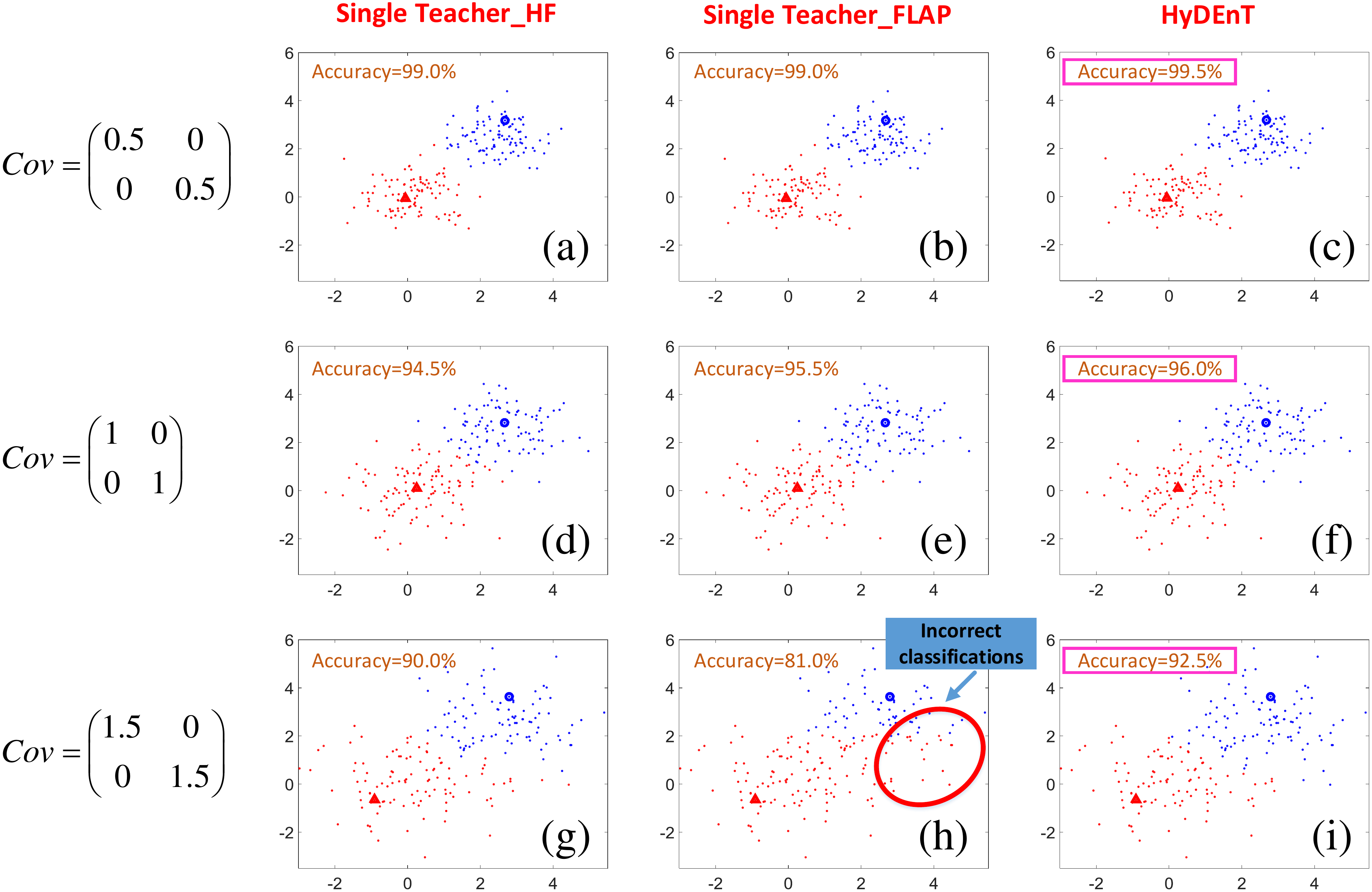}\\
  \caption{Comparison of single-teacher and multi-teacher settings on \emph{NoisyGaussian} dataset under different covariances. Red and blue dots denote positive
  and negative examples, respectively. Red triangle and blue circle represent labeled positive example and labeled negative example accordingly. The left, middle, and right columns
  present the results of ``Single Teacher\_HF'', ``Single Teacher\_FLAP'', and the proposed ``HyDEnT''. Different rows correspond to different levels of noise controlled by the increased covariance.}\label{fig:SingleTeacherVSMultiTeacher}
  \vspace{-7pt}
\end{figure*}

One feature of our HyDEnT method is that multiple teachers are introduced to form a teaching ensemble, which is better than simply employing single teacher. This makes HyDEnT stably yield satisfactory performance, especially in the presence of noisy dataset.\par
We synthesize a \emph{NoisyGaussian} dataset which is composed of two data clusters generated from two Gaussian distributions centered at $(0,0)$ and $(2.5,2.5)$, respectively. Each Gaussian forms a class, and only one example in each class is regarded as labeled (see Fig.~\ref{fig:SingleTeacherVSMultiTeacher}). In this experiment, we gradually add noise to the dataset by increasing the covariance of two Gaussians, and then evaluate the robustness of HyDEnT and different single teacher settings. Specifically, we employ Eq.~\eqref{eq:9} as teaching model to respectively guide the propagation process of HF and FLAP, and thus the two settings with single teacher are denoted by ``Single Teacher\_HF'' and ``Single Teacher\_FLAP'' accordingly.\par 

From the results shown in Fig.~\ref{fig:SingleTeacherVSMultiTeacher}, we observe that our ensemble teaching strategy (\emph{i.e.} HyDEnT) consistently generates
higher accuracy than the two single teacher settings with the noise level ranging from low to high. When the dataset is relatively clean, namely the covariance matrix $Cov=\left(\begin{smallmatrix}0.5 & 0 \\ 0 & 0.5\end{smallmatrix}\right)$, it can be found that both single teacher settings and
multi-teacher ensemble are able to achieve almost 100\% accuracy. When more noises are introduced by the covariance matrix $Cov=\left(\begin{smallmatrix}1 & 0 \\ 0 & 1\end{smallmatrix}\right)$,
we see that the accuracies of ``Single Teacher\_HF'', ``Single Teacher\_FLAP'' and ``HyDEnT'' slightly drop by 4.5\%, 3.5\% and 3.5\%, correspondingly.
However, if the dataset is contaminated by much heavy noise as indicated in the last row of Fig.~\ref{fig:SingleTeacherVSMultiTeacher}, we see that the accuracy of
``Single Teacher\_FLAP'' significantly descends to 81.0\%. Such unsatisfactory result happens because the
propagation process is misled by the outliers located between the two classes, therefore many negative examples are erroneously classified as positive as indicated
by the red circle in Fig.~\ref{fig:SingleTeacherVSMultiTeacher}(h). This reveals that only employing single teacher-learner pair is not safe for accurate propagation. In contrast, HyDEnT with ensemble teaching and
hybrid label propagation generates 92.5\% accuracy although the noisy data poses a great challenge for obtaining
good result. From this sense, we know that incorporating an ensemble of multiple teachers and learners helps to render
stable propagation results.

\subsection{Comparison with Other Methods}
\label{sec:Comparison}
\begin{figure}
  \centering
  \includegraphics[width=0.9\linewidth]{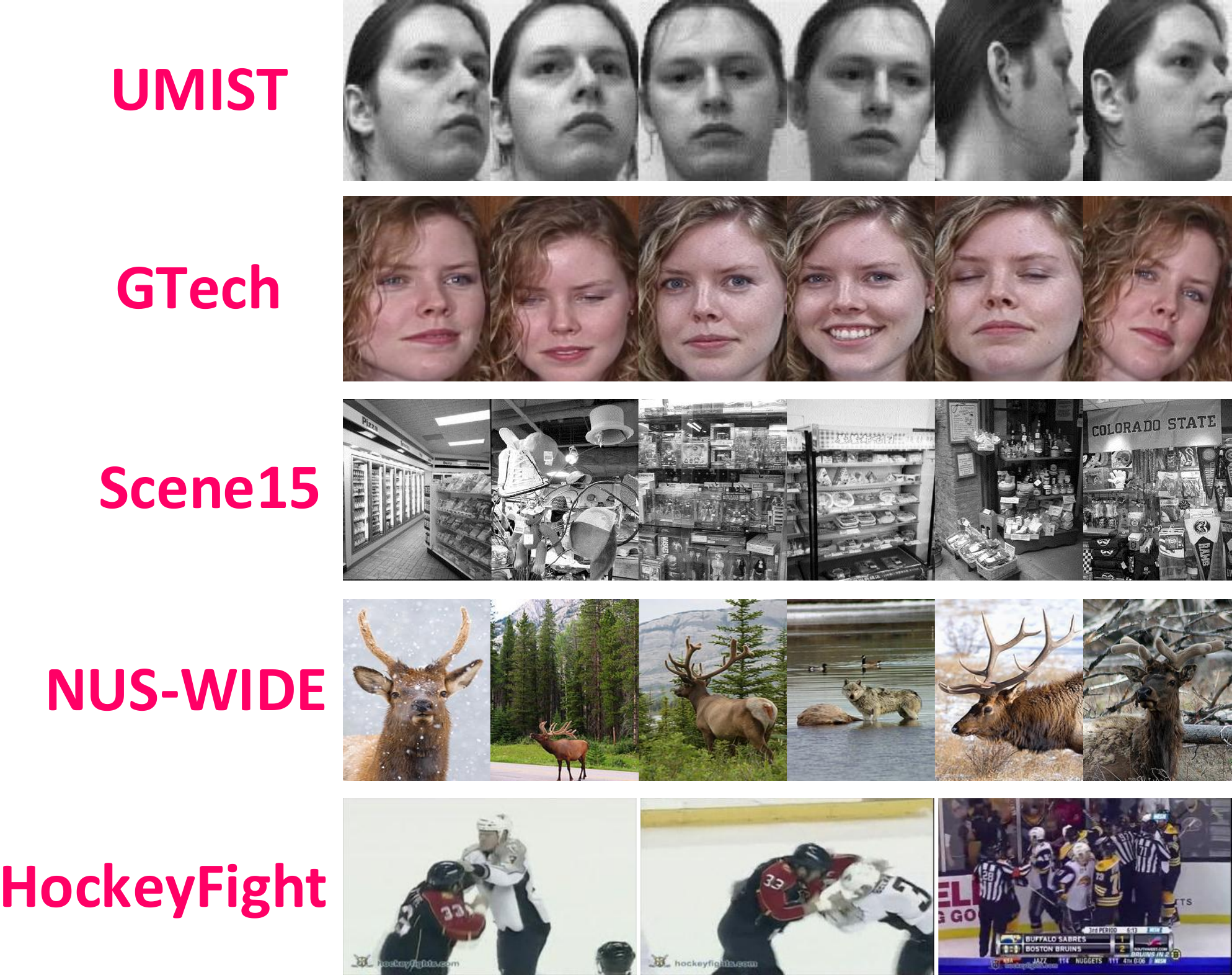}\\
  \caption{Some example images or frames in the adopted datasets including \emph{UMIST}, \emph{GTech}, \emph{Scene15}, \emph{NUS-WIDE}, and \emph{HockeyFight}. Each row displays the examples belonging to one class in the corresponding dataset.}\label{fig:ExampleImages}
  \vspace{-10pt}
\end{figure}

In this section, we compare the proposed HyDEnT algorithm with six representative propagation methods on six popular datasets including \emph{UMIST} \cite{Graham1995UMIST}, \emph{GTech}\footnote{\url{http://www.anefian.com/face reco.htm}}, \emph{Scene15} \cite{lazebnik2006beyond}, \emph{NUS-WIDE} \cite{chua2009nus}, \emph{HockeyFight}\footnote{\url{http://visilab.etsii.uclm.es/personas/oscar/FightDetection/index.html}}, and \emph{ISOLET} \cite{Fanty1990Spoken}. Among them, the \emph{UMIST} and \emph{GTech} datasets are on face recognition, \emph{Scene15} is on scene categorization, \emph{NUS-WIDE} \cite{chua2009nus} focuses on general image classification, \emph{HockeyFight} is for violent behaviour detection, and \emph{ISOLET} studies spoken letter recognition. The attributes of the adopted six datasets are summarized in Tab.~\ref{tab:Datasets}, from which we can see that the employed datasets cover a wide range of example quantity, category quantity, investigated tasks, and example types. Some example images/frames in the first five datasets are provided in Fig.~\ref{fig:ExampleImages}.\par

\begin{table*}[t]
\caption{An overview of the adopted datasets.}\label{tab:Datasets}
\begin{center}
\begin{tabular}{lcccc}
\toprule
             &  \# examples & \# classes & task & type \\
\midrule
\emph{UMIST}       & 575        & 20         & face recognition & image\\
\emph{GTech}       & 750        & 50         & face recognition & image\\
\emph{Scene15}     & 4485       & 15         & scene categorization & image\\
\emph{NUS-WIDE}    & 47254      & 112        & general image classification & image\\
\emph{HockeyFight} & 1000       & 2          & violent behaviour detection & video\\
\emph{ISOLET}      & 7800       & 26         & spoken letter recognition & audio\\
\bottomrule
\end{tabular}
\end{center}
\vskip -15pt
\end{table*}

The six baselines include classical propagation methods Harmonic Functions (HF) \cite{Zhu2002Report} and Linear Neighborhood Propagation (LNP) \cite{wang2009linear}, state-of-the-art method Fick's Law Assisted Propagation (FLAP) \cite{Gong2015Fick}, recent graph-based ensemble methodologies Semi-Supervised Ensemble Learning (SSEL) \cite{woo2012semi} and Bipartite Graph-based Consensus Maximization
(BGCM) \cite{gao2009graph}, and the most relevant Teaching-to-Learn and Learning-to-Teach (TLLT) \cite{Gong2016TLLT} that introduces single teacher for label propagation. Note that HF and FLAP are also the two learners in the implementation of our HyDEnT algorithm, so comparing them with HyDEnT helps to see the effect brought by the proposed teaching method. Besides, since BGCM is an ensemble algorithm combining the outputs from multiple models, we also use HF and FLAP as its two base models to achieve fair comparison.

\begin{figure*}
  \centering
  \includegraphics[width=0.9\linewidth]{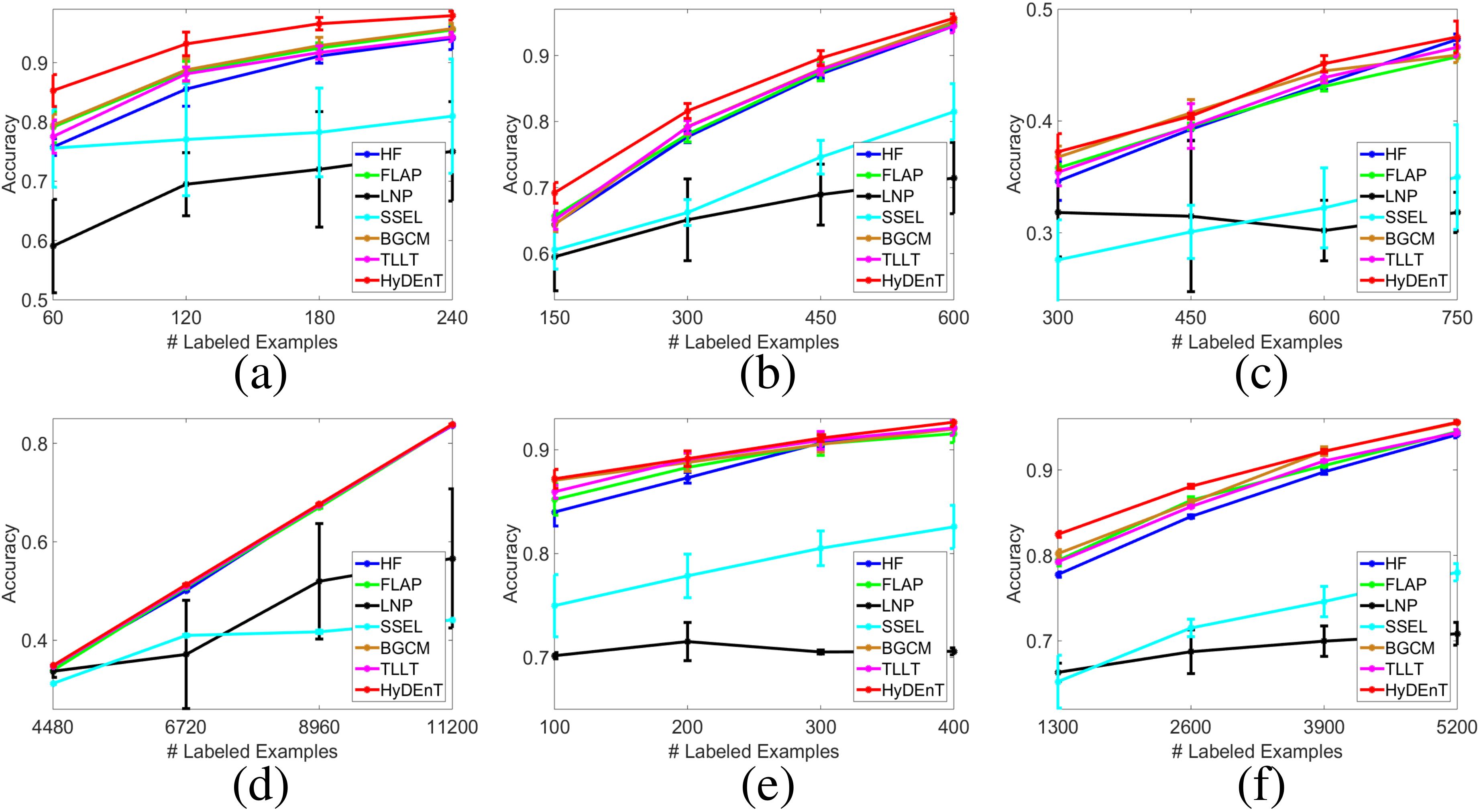}\\
  \caption{The classification accuracies of all the compared methods on six popular datasets regarding image, video, and audio. (a) is \emph{UMIST}, (b) is \emph{GTech}, (c) is \emph{Scene15}, (d) is \emph{NUS-WIDE}, (e) is \emph{HockeyFight}, and (d) is \emph{ISOLET}.}\label{fig:FaceComparison}
  \vspace{-12pt}
\end{figure*}

\subsubsection{\textbf{UMIST}}
The \emph{UMIST} face recognition dataset consists of totally 575 face images belonging to 20 individuals with different races, genders, and appearances. In our experiments, we use the cropped $112\times92$ face images\footnote{\url{http://www.cs.nyu.edu/~roweis/data.html}.} to compare the recognition accuracies of HF, FLAP, LNP, SSEL, BGCM, TLLT, and HyDEnT. Specifically, we randomly select 3, 6, 9, 12 images of each individual as labeled examples, and take the remaining examples as unlabeled ones.\par
Throughout this paper, we build the identical 5-NN graph for HF, FLAP, BGCM, TLLT and HyDEnT on each dataset, because it has been widely observed that a sparse graph usually leads to satisfactory performance \cite{Zhuang2015Constructing,Chen2013Large}. The number of neighbors in the graphs for LNP and SSEL is set to 10, as they operate on a different graph from HF, FLAP, BGCM, TLLT and HyDEnT. The trade-off parameter $\alpha$ in BGCM is tuned to 1 as suggested by \cite{gao2009graph}. The learning rate $\gamma$ for both TLLT and HyDEnT is tuned to 0.5. Similar to Section \ref{sec:VerificationOfOurMethod}, here we also use the pixel-wise gray value feature to characterize each image, and observe
the classification accuracies of all the compared methods with different sizes of labeled sets.\par

Every algorithm is independently implemented ten times, and the reported accuracies and standard deviations are the mean values of the outputs of these ten independent implementations. The performances of all the algorithms are presented in Fig.~\ref{fig:FaceComparison}(a), which suggests that our HyDEnT achieves the top level performance among all the compared methods. BGCM comes in the second place, and its accuracies are lower than HyDEnT with margins approximately 6\%, 5\%, 4\% and 2\% when $l=$ 60, 120, 180 and 240, respectively. FLAP and HF also perform worse than HyDEnT, so our proposed HyDEnT can improve the results of either FLAP or HF by properly combining their advantages in an ensemble teaching way. Furthermore, SSEL and LNP achieve significantly worse results than HyDEnT, and their standard deviations are also quite large. This indicates that SSEL and LNP are very sensitive to the choice of initial labeled examples. Comparatively, the error bars of HyDEnT indicate that its standard deviations are very small, reflecting that HyDEnT is stable and is able to obtain impressive performance regardless of the initial labeled examples.

\subsubsection{\textbf{GTech}}
\emph{GTech} face database contains the images of 50 people taken at the Center for Signal and Image Processing at Georgia Institute of Technology. Each people has 15 color images with cluttered background, and these 15 images cover frontal and/or tilted faces with different facial expressions, lighting conditions and scales. Here we use the cropped face images for our experiments, which are further resized to the resolution of $40\times30$.\par


The accuracies of all the methods are particularly investigated when $l$ varies from 150 to 600. The comparison results presented in Fig.~\ref{fig:FaceComparison}(b) indicate that our HyDEnT is in the first place, which is followed by TLLT and BGCM. Specifically, HyDEnT outperforms TLLT with a margin 4\%, 3\%, 2\%, 1\% when $l=$ 150, 300, 450, 600, respectively. Besides, it can be observed that HyDEnT is significantly better than each of its two learners HF and FLAP, which also demonstrates the strength of HyDEnT.

\subsubsection{\textbf{Scene15}}
\emph{Scene15} dataset contains the images of fifteen natural scene categories including bedroom, kitchen, street, store, and so on. Each category has 200$\sim$400 images, and average image size is $300 \times 250$ pixels. In our experiment, every image is represented by a 72-dimensional Pyramid Histogram Of Gradients (PHOG) feature vector \cite{bosch2007representing}, and our task is to identify which of the fifteen scene categories it belongs to.\par

We compare the performances of various algorithm when the number of originally labeled examples $l$ changes from 300 to 750. Fig.~\ref{fig:FaceComparison}(c) shows the results. Since scene recognition is a very challenging task, all the compared methods obtain relatively low classification accuracy on this dataset. Among the comparators, we can see that HF, FLAP and TLLT achieve comparable performances. In contrast, the proposed HyDEnT consistently beat all the baseline methods under different selections of $l$. Specifically, it can be noted that the accuracies generated by our HyDEnT are higher than other ensemble approaches like SSEL and BGCM, so HyDEnT can effectively exploit the advantage of each of the adopted base classifiers.

\subsubsection{\textbf{NUS-WIDE}}
The \emph{NUS-WIDE} is a web image dataset created by Lab for Media Search in National University of Singapore. In this dataset, the groudtruth label of every image example has been manually annotated, and thus this dataset is utilized here to evaluate the capability of an algorithm on image classification. For our experiment, we only reserve the classes that have more than 100 images, so a subset of \emph{NUS-WIDE} containing 47254 images with 112 classes is obtained.\par

Similar to the experiments on \emph{Scene15} dataset, here we also use the 72-dimensional PHOG feature vector to represent an image, and implement the compared methods including HF, FLAP, LNP, SSEL, BGCM, TLLT and HyDEnT for ten times with different choices of initially labeled examples. The experimental results are provided in Fig.~\ref{fig:FaceComparison}(d), from which we see that HF, FLAP, BGCM, TLLT and HyDEnT achieve very similar performance. 
A notable fact is that when the number of labeled examples is 11200, the classification accuracy of our HyDEnT on the remaining unlabeled images is as high as 83.87\%, and this is a very impressive performance since \emph{NUS-WIDE} is a very challenging dataset focusing on general image classification.

\subsubsection{\textbf{HockeyFight}}
In this section, we show that our method can also be applied to video analysis. Detecting the violent behaviors such as fighting and robbery is an important task in video surveillance, so the \emph{HockeyFight} dataset is employed here to test the various compared methods on recognizing the fighting behaviors during the hockey match. This dataset is made up of 1000 video clips collected in hockey competitions, among which 500 contain fighting behavior and 500 are non-fighting clips. Therefore, the task of this dataset is to correctly identify whether a video clip contains fighting behavior or not. As suggested by \cite{Xu2014Violent}, we utilize the Motion SIFT (MoSIFT) processed by Kernel Density Estimation (KDE) to form the action descriptors, and then use the Bag-of-Words (BoW) approach to characterize each video clip as a histogram over 100 visual words. As a result, every video clip in this dataset is represented by a 100-dimensional feature vector.\par
The experimental results are displayed in Fig.~\ref{fig:FaceComparison}(e). We can see that HF, FLAP, BGCM, TLLT and HyDEnT generally perform favorably to SSEL and LNP. The classification accuracies of HF are 84.00\%, 87.30\%, 90.65\% and 92.05\% when $l=$ 100, 200, 300 and 400, respectively. After imposing a single teacher on HF as suggested by TLLT, the accuracies can be enhanced to 85.95\%, 89.06\%, 90.90\% and 92.10\%. However, if multiple teachers are incorporated as the proposed HyDEnT, we see that the performance can be further improved to 87.20\%, 89.14\%, 91.12\% and 92.66\%. This reflects that the teaching committee and hybrid label propagation inherited by HyDEnT plays an important role in boosting the classification accuracy.

\subsubsection{\textbf{ISOLET}}
In this section, we study the ability of HF, FLAP, LNP, SSEL, BGCM, TLLT and HyDEnT on English spoken letter recognition. To this end, the \emph{ISOLET} database is adopted which consists of 7800 spoken letters ``A''$\sim$``Z'' produced by 150 male and female speakers.\par

The accuracies averaged over ten implementations obtained by HyDEnT and other baseline algorithms are displayed in Fig.~\ref{fig:FaceComparison}(f). We see that the performances of all the compared methods can be improved when the number of labeled examples gradually changes from small to large. The accuracy obtained by HF ranges from 77.76\% to 94.13\%, which is slightly worse than FLAP and TLLT with the accuracies being 79.40\%$\sim$94.49\% and 79.24\%$\sim$94.37\%, respectively. Although FLAP and TLLT have gained very encouraging results, the proposed HyDEnT are still able to improve their performances and its accuracies are 82.46\%$\sim$95.58\%. Specifically, we note that when the available labeled examples are very scarce (\emph{e.g.} $l=1300$), the advantage of our HyDEnT becomes more obvious than some very competitive existing methods such as FLAP, BGCM, and TLLT. Such superiority is due to the fact that HyDEnT arranges more suitable curriculums for propagation than other methods, and this is very important when we have to classify a large number of unlabeled examples given very few labeled examples. 

\subsubsection{\textbf{Significance Test}}
Above experiments have empirically shown that the proposed HyDEnT performs better than the compared baselines in most cases. In this section we use the paired t-test 
to statistically demonstrate such superiority of HyDEnT to other methods.\par

The paired t-test is a statistical tool to determine whether two sets of observations are essentially the same. In our experiments, all the compared methods are independently implemented ten times on different $l$ for the six datasets, so we may use paired t-test to examine whether the ten accuracies output by HyDEnT are significantly higher than those generated by the comparator. Table~\ref{Table t} lists the test results of HyDEnT versus every baseline algorithm, which suggests that HyDEnT is significantly better than other methods in most situations. 

\begin{table}[t]
\setlength\tabcolsep{3.5pt}
\caption{Paired t-test of the proposed HyDEnT to the compared algorithms (confidence level: 0.9). ``$\surd$'' indicates that HyDEnT is significantly better than the corresponding method, and ``-'' means that the performances of HyDEnT and the corresponding method are comparable.) \label{Table t}}
\small
\begin{center}
\begin{tabular}{|c|c|c|c|c|c|c|c|}
\hline
             & $l$ & HF & FLAP & LNP & SSEL & BGCM & TLLT \\
\hline\hline
\multirow{4}*{\textit{UMIST}}
             & 60  &  $\surd$  &  $\surd$    &  $\surd$      &  $\surd$ & $\surd$  & $\surd$  \\
             & 120 &  $\surd$  &  $\surd$    &  $\surd$      &  $\surd$ & $\surd$  & $\surd$  \\
             & 180 &  $\surd$  &  $\surd$    &  $\surd$      &  $\surd$ & $\surd$  & $\surd$  \\
             & 240 &  $\surd$  &  $\surd$    &  $\surd$      &  $\surd$ & $\surd$  & $\surd$ \\

\hline
\multirow{4}*{\textit{GTech}}
             & 150 & $\surd$   &  $\surd$    &  $\surd$    & $\surd$  & $\surd$  &  $\surd$ \\
             & 300 & $\surd$   &  $\surd$    &  $\surd$    & $\surd$  & $\surd$  &  $\surd$ \\
             & 450 & $\surd$   &  $\surd$    &  $\surd$    & $\surd$  & $\surd$  &  $\surd$ \\
             & 600 & $\surd$   &  $\surd$    &  $\surd$    & $\surd$  & -  &  $\surd$ \\

\hline
\multirow{4}*{\textit{Scene15}}
             & 300  &  $\surd$  &   $\surd$  &   $\surd$  & $\surd$  & -        & $\surd$  \\
             & 450  &  -        &   -        &   $\surd$  & $\surd$  & -        & $\surd$  \\
             & 600  &  $\surd$  &   $\surd$  &   $\surd$  & $\surd$  & $\surd$  & $\surd$  \\
             & 750  &  -        &   $\surd$  &   $\surd$  & $\surd$  & $\surd$  & $\surd$  \\

\hline
\multirow{4}*{\textit{NUS-WIDE}}
             & 4480  & -         & -       &  $\surd$  & $\surd$  & -  & -  \\
             & 6720  & $\surd$   & -       &  $\surd$  & $\surd$  & -  & -  \\
             & 8960  & -         & -       &  $\surd$  & $\surd$  & -  & -  \\
             & 11200 & -         & -       &  $\surd$  & $\surd$  & -  & -  \\

\hline
\multirow{4}*{\textit{HockeyFight}}
             & 100   &  $\surd$  & $\surd$ &  $\surd$      & $\surd$  & -       &  $\surd$ \\
             & 200   &  $\surd$  & $\surd$ &  $\surd$      & $\surd$  & -       &  - \\
             & 300   &  -        & -       &  $\surd$      & $\surd$  & $\surd$ &  - \\
             & 400   &  -        & $\surd$ &  $\surd$      & $\surd$  & -       &  - \\
\hline
\multirow{4}*{\textit{ISOLET}}
             & 1300   & $\surd$   & $\surd$    & $\surd$   & $\surd$  & $\surd$  & $\surd$  \\
             & 2600   & $\surd$   & $\surd$    & $\surd$   & $\surd$  & $\surd$  & $\surd$  \\
             & 3900   & $\surd$   & $\surd$    & $\surd$   & $\surd$  & -  & $\surd$  \\
             & 5200   & $\surd$   & $\surd$    & $\surd$   & $\surd$  & -  &  $\surd$ \\

\hline
\end{tabular}
\end{center}
\vskip -15pt
\end{table}

\subsection{Illustration of Convergence}
\label{sec:IllustrationOfConvergence}

\begin{figure}
  \centering
  \includegraphics[width=0.95\linewidth]{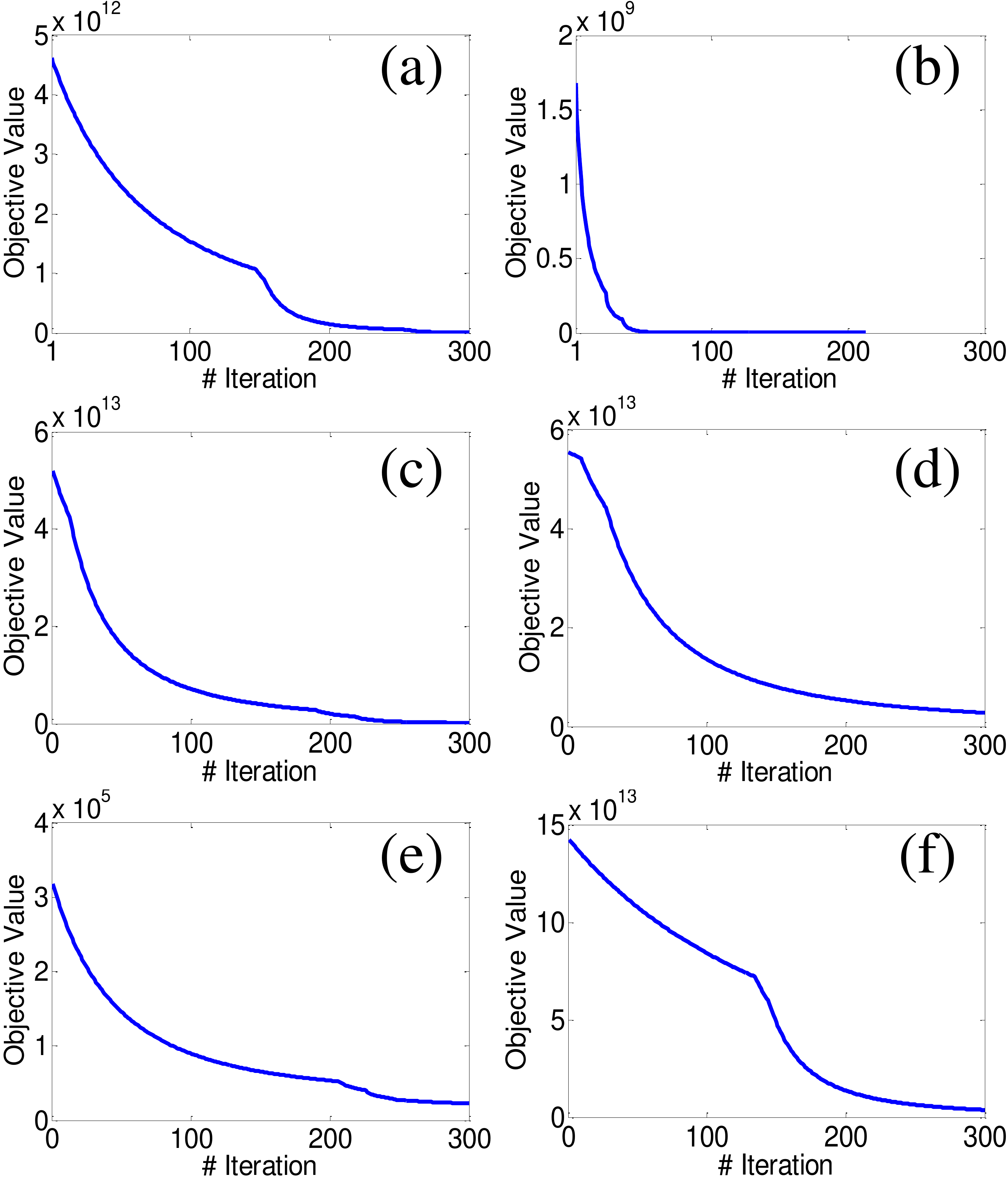}\\
  \caption{The convergence process of the BCD method adopted by our HyDEnT algorithm. (a) is \emph{UMIST} dataset, (b) is \emph{GTech}, (c) is \emph{Scene15} dataset, (d) is \emph{NUS-WIDE} dataset, (e) is \emph{HockeyFight}, and (f) is \emph{ISOLET} dataset.}\label{fig:ConvergenceCurve}
  \vspace{-6pt}
\end{figure}

In Section~\ref{sec:ProofOfConvergence}, we have theoretically proved that the designed iterative BCD optimization process monotonically decreases the objective function in Eq.~\eqref{eq:14} and finally converges to a stationary point. Here we plot the variation of objective values (\emph{i.e.} $Q(\bar{\mathbf{S}})$ in Eq.~\eqref{eq:14}) when the iteration proceeds on the above six datasets including \emph{UMIST}, \emph{GTech}, \emph{Scene15}, \emph{NUS-WIDE}, \emph{HockeyFight}, and \emph{ISOLET}. From the curves in Fig.~\ref{fig:ConvergenceCurve}, we see that the objective value decrease rapidly on all the six adopted datasets. This observation coincides with our previous theoretical findings and demonstrates that BCD is effective for solving the ensemble teaching model in Eq.~\eqref{eq:14}.

\section{Conclusion}
\label{sec:conclusion}
This paper proposed an ensemble teaching algorithm for hybrid label diffusion. The teaching algorithm is formulated as an optimization problem, which explicitly considers both the individuality of each teacher and the shared common knowledge among different teachers. Consequently, the incorporated teachers are able to cooperate with each other to pick up the overall simplest curriculum examples. Due to the efforts of the teaching committee, all the unlabeled examples are logically ``learned'' (\emph{i.e.} classified) by different learners (\emph{i.e.} propagation algorithms) via a simple-to-difficult order, and thus the propagation quality can be improved. The experimental results on several typical datasets reveal that the proposed approach obtains higher classification accuracy than existing state-of-the-art propagation methods, and this validates the rationality and effectiveness of our ensemble teaching strategy. 


%

%
%
%
%
%

\ifCLASSOPTIONcaptionsoff
  \newpage
\fi

{
\bibliographystyle{IEEEtran}
\bibliography{EnsembleTLLT}
}


%
%
%
%
%
%
%

\end{document}